\documentclass[letterpaper]{article} % DO NOT CHANGE THIS
\usepackage{aaai24}  % DO NOT CHANGE THIS
\usepackage{times}  % DO NOT CHANGE THIS
\usepackage{helvet}  % DO NOT CHANGE THIS
\usepackage{courier}  % DO NOT CHANGE THIS
\usepackage[hyphens]{url}  % DO NOT CHANGE THIS
\usepackage{graphicx} % DO NOT CHANGE THIS
\usepackage{natbib}  % DO NOT CHANGE THIS AND DO NOT ADD ANY OPTIONS TO IT
\usepackage{caption} % DO NOT CHANGE THIS AND DO NOT ADD ANY OPTIONS TO IT
\usepackage{amsmath,amsfonts,bm}

\def\eqref#1{equation~\ref{#1}}
\def\1{\bm{1}}

\DeclareMathAlphabet{\mathsfit}{\encodingdefault}{\sfdefault}{m}{sl}
\SetMathAlphabet{\mathsfit}{bold}{\encodingdefault}{\sfdefault}{bx}{n}

\usepackage{booktabs}       % professional-quality tables
\usepackage{multirow}

\newcommand{\mr}[2]{\multirow{#1}{*}{#2}}
\newcommand{\mc}[3]{\multicolumn{#1}{#2}{#3}}
\usepackage{amsfonts}       % blackboard math symbols
\usepackage{nicefrac}       % compact symbols for 1/2, etc.
\usepackage{microtype}      % microtypography
\usepackage{array}
\newcolumntype{P}[1]{>{\centering\arraybackslash}p{#1}}
\usepackage{url}            % simple URL typesetting

\usepackage{amsmath,amsthm} % define this before the line numbering.
\usepackage{xspace}
\usepackage{enumerate}
\usepackage{enumitem}
\usepackage{stmaryrd}
\usepackage{caption}
\usepackage{subcaption}
\usepackage{graphicx} % more modern
\usepackage{algorithm}
\usepackage{algorithmic}
\usepackage{sidecap}
\usepackage{colortbl}
\usepackage{tikz}
\usepackage{listings}
\usepackage{xcolor}
\usepackage{makecell}
\usepackage{titletoc}
\usepackage{amsmath}
\usepackage{adjustbox}

\DeclareCaptionStyle{ruled}{labelfont=normalfont,labelsep=colon,strut=off} % DO NOT CHANGE THIS
\floatstyle{ruled}
\newfloat{listing}{tb}{lst}{}
\floatname{listing}{Listing}
\newcount\COMMENTs  % 0 suppresses notes to selves in text
\usepackage{color}
\definecolor{darkgreen}{rgb}{0,0.5,0}
\definecolor{purple}{rgb}{1,0,1}
\newcommand{\comm}[2]{\ifnum\COMMENTs=1\textcolor{#1}{#2}\fi}
\newcommand{\xhdr}[1]{{\noindent\bfseries #1}.}

\definecolor{dkred}{rgb}{0.5,0,0}
\definecolor{dkgreen}{rgb}{0,0.6,0}
\definecolor{gray}{rgb}{0.5,0.5,0.5}
\definecolor{mauve}{rgb}{0.58,0,0.82}

\newcommand{\ie}{\textit{i.e.}}
\newcommand{\eg}{\textit{e.g.}}

\definecolor{customgray}{rgb}{0.3,0.3,0.3}
\definecolor{customgreen}{RGB}{140,211,89}
\newcommand{\std}[1]{\textcolor{customgray}{\scriptsize{$\pm$#1}}}

\usepackage{enumitem}   
\usepackage{amsthm}

\newcommand{\Lcal}{\mathcal{L}}

\newcommand{\Xcal}{\mathcal{X}}

\newcommand{\Zcal}{\mathcal{Z}}
\newcommand{\EE}{\mathbb{E}} % Expectation
\newcommand{\PP}{\mathbb{P}} % Probability
\newcommand{\RR}{\mathbb{R}} % Real numbers
\newcommand{\signn}{\mathop{\mathrm{sign}}}
\newcommand{\one}{\mathbf{1}}  % Identity
\ifx\proof\undefined
\newenvironment{proof}{\par\noindent{\bf Proof\ }}{\hfill\BlackBox\\[2mm]}
\fi
\theoremstyle{plain}
\ifx\theorem\undefined
\newtheorem{theorem}{Theorem}

\newtheorem{lemma}{Lemma}

\theoremstyle{definition}

\fi

\newcommand{\htheta}{\hat \theta}

\newcommand{\by}{\bar y}
\newcommand{\bx}{\bar x}

\newcommand{\tx}{\tilde x}

\newcommand{\ty}{\tilde y}

\newcommand{\tell}{\tilde \ell}
\newcommand{\tLcal}{\tilde \Lcal}
\newcommand{\bz}{\bar z}

\newcommand{\methodname}{\textsc{TuneUp}}

\title{\methodname{}: A Simple Improved Training Strategy for Graph Neural Networks}
\author{
    Weihua Hu\textsuperscript{\rm 1}, Kaidi Cao\textsuperscript{\rm 1}, Kexin Huang\textsuperscript{\rm 1}, Edward W Huang\textsuperscript{\rm 2}, \\ Karthik Subbian\textsuperscript{\rm 2}, Kenji Kawaguchi\textsuperscript{\rm 3}, Jure Leskovec\textsuperscript{\rm 1}.
}
\affiliations{
    \textsuperscript{\rm 1}Stanford University, \textsuperscript{\rm 2}Amazon, \textsuperscript{\rm 3}National University of Singapore \\
    \{weihuahu,kaidicao,kexinh,jure\}@cs.stanford.edu, \ \{ksubbian,ewhuang\}@amazon.com,\ kenji@comp.nus.edu.sg
}

\begin{document}
\maketitle

\begin{abstract}
Despite recent advances in Graph Neural Networks (GNNs), their training strategies remain largely under-explored.
The conventional training strategy learns over all nodes in the original graph(s) equally, which can be sub-optimal as certain nodes are often more difficult to learn than others.
Here we present \methodname{}, a simple curriculum-based training strategy for improving the predictive performance of GNNs.
\methodname{} trains a GNN in two stages.
In the first stage, \methodname{} applies conventional training to obtain a strong base GNN.
The base GNN tends to perform well on head nodes (nodes with large degrees) but less so on tail nodes (nodes with small degrees). Therefore, the second stage of \methodname{} focuses on improving prediction on the difficult tail nodes by further training the base GNN on synthetically generated tail node data.
We theoretically analyze \methodname{} and show it provably improves generalization performance on tail nodes.
\methodname{} is simple to implement and applicable to a broad range of GNN architectures and prediction tasks.
Extensive evaluation of \methodname{} on five diverse GNN architectures, three types of prediction tasks, and both transductive and inductive settings shows that \methodname{} significantly improves the performance of the base GNN on tail nodes, while often even improving the performance on head nodes. Altogether, \methodname{} produces up to 57.6\% and 92.2\% relative predictive performance improvement in the transductive and the challenging inductive settings, respectively.
\end{abstract}

% <concept_id>10002951.10003317.10003347.10003350</concept_id>
% <concept_desc>Information systems~Recommender systems</concept_desc>

% \ccsdesc[500]{Information systems~Recommender systems}

% \keywords{graph neural networks, tail nodes}

\maketitle

\section{Introduction}
\label{sec:intro}
Graph Neural Networks (GNNs) are one of the most successful and widely used paradigms for representation learning on graphs, achieving state-of-the-art performance on a variety of prediction tasks, such as semi-supervised node classification~\citep{kipf2016semi,velivckovic2017graph}, link prediction~\citep{hamilton2017inductive,kipf2016variational}, and recommender systems~\citep{ying2018graph,he2020lightgcn}.
There has been a surge of work on improving GNN model architectures~\citep{velivckovic2017graph,xu2018how,xu2018representation,shi2020masked,klicpera2018predict,wu2019simplifying,zhao2019pairnorm,li2019deepgcns,chen2020simple,li2021training} and task-specific losses~\citep{kipf2016variational,rendle2012bpr,verma2021graphmix,huang2021mixgcf}.
Despite all these advances, strategies for training a GNN on a given supervised loss remain largely under-explored.
Existing work has focused on minimizing the given loss averaged over nodes in the original graph(s), which neglects the fact that some nodes are more difficult to learn than others.

\begin{figure*}
\centering
\includegraphics[width=\linewidth]{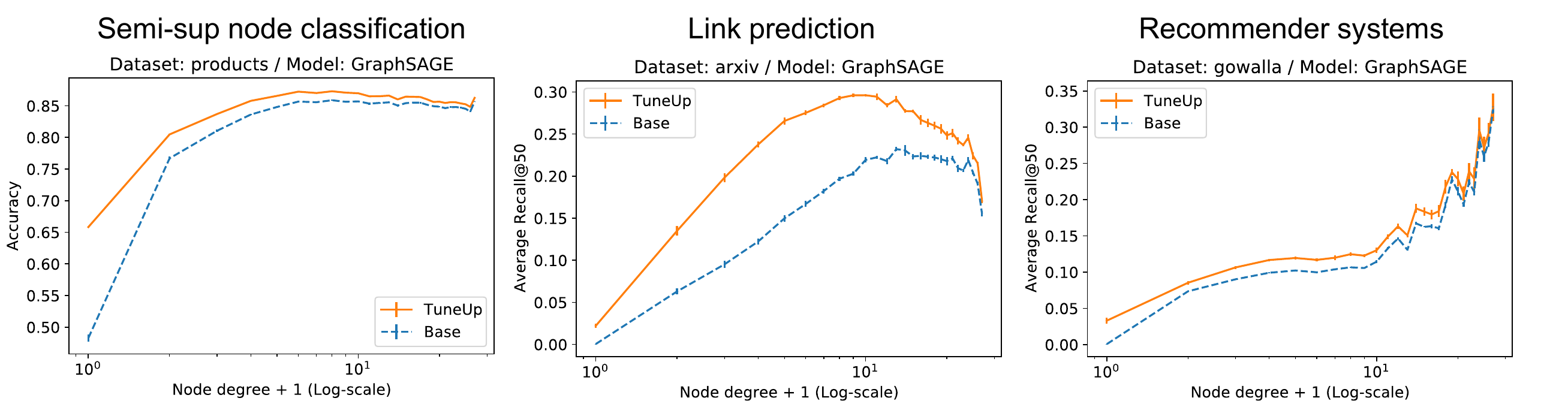}
\caption{Degree-specific predictive performance of the base GNN (trained with conventional training) and \methodname{} GNN in the transductive setting. Between ``Base'' and ``\methodname{}'', only the training strategy differs; \emph{the model architecture and the loss function stay exactly the same}.
The $x$-axis represents the node degrees in the training graph, and the $y$-axis is the predictive performance averaged over nodes with specific degrees. We see from the dotted blue curves that the base GNN tends to perform poorly on tail nodes, \ie, nodes with small degrees. Our \methodname{} (denoted by the solid orange curves) gives more accurate GNNs than conventional training (``Base''). \methodname{} improves the predictive performance across almost all node degrees, but most significantly on tail nodes. 
}
\label{fig:base_comp}
\end{figure*}

Here we present \methodname{}, a simple improved training strategy for improving the predictive performance of GNNs.
The key motivation behind \methodname{} is that GNNs tend to under-perform on tail nodes, \ie, nodes with small (\eg, 0--5) node degrees, due to the scarce neighbors to aggregation features from~\citep{liu2021tail}. Improving GNN performance on tail nodes is important since they are prevalent in real-world scale-free graphs~\citep{clauset2009power} as well as newly arriving \emph{cold-start} nodes~\citep{lika2014facing}. 

The key idea of \methodname{} is to adopt curriculum-based training~\cite{bengio2009curriculum}, where it first trains a GNN to perform well on relatively-easy head nodes. It then proceeds to further train the GNN to perform well on the more difficult tail nodes by minimizing the loss over supervised tail node data that is synthetically generated via pseudo-labeling and dropping edges~\citep{rong2019dropedge}.

% Concretely, \methodname{} trains a GNN in two stages. In the first stage, \methodname{} employs the conventional training strategy, \ie, simply minimizing the given supervised loss over nodes in the original graph(s), to produce a strong base GNN to start with.
% The base GNN tends to perform well on head nodes, \ie, nodes with a large number (\eg, 6--) of neighbors, but poorly on tail nodes (see the dotted blue curves in Figure~\ref{fig:base_comp}).
% To mitigate this issue, the second stage of \methodname{} focuses on improving the performance on the difficult tail nodes. Specifically, \methodname{} further trains the base GNN by minimizing the loss over supervised tail node data that is synthetically generated via pseudo-labeling and dropping edges~\citep{rong2019dropedge}.
% The dedicated training on the synthetic tail nodes allows the resulting GNN to perform much better on the real tail nodes. 

We theoretically analyze \methodname{} and show that it provably improves the generalization performance on tail nodes by utilizing information from head nodes. Our theory also justifies how \methodname{} generates supervised synthetic tail nodes in the second stage. Our theory suggests that both pseudo-labeling and dropping edges are crucial for improved generalization.

\methodname{} is simple to implement on top of the conventional training pipeline of GNNs, as shown in Algorithm~\ref{alg:method}. 
Thanks to its simplicity, \methodname{} can be readily used with a wide range of GNN models and supervised losses; hence, applicable to many node and edge-level prediction tasks.
This is in contrast with recent methods for improving GNN performance on tail nodes~\citep{liu2021tail,zheng2021cold,zhang2021graph,kang2022rawlsgcn} as they all require non-trivial modifications of both GNN architectures and loss, making them harder to implement and not applicable to diverse prediction tasks.

To demonstrate the effectiveness and broad applicability of \methodname{}, we perform extensive experiments on a wide range of settings. 
We consider five diverse GNN architectures, three types of key prediction tasks (semi-supervised node classification, link prediction, and recommender systems) with a total of six datasets, as well as both transductive (\ie, prediction on nodes seen during training) and inductive (\ie, prediction on new nodes not seen during training) settings. For the inductive setting, we additionally consider the challenging cold-start scenario (\ie, limited edge connectivity from new nodes) by randomly removing certain portions of edges from new nodes. 

Across the settings, \methodname{} produces consistent improvement in the predictive performance of GNNs.
In the transductive setting, \methodname{} significantly improves the performance of base GNNs on tail nodes, while oftentimes even improving the performance on head nodes (see Figure~\ref{fig:base_comp}).
In the inductive setting, \methodname{} especially shines in the cold-start prediction scenario, where new nodes are tail-like, producing up to 92.2\% relative improvement in the predictive performance.
Moreover, our \methodname{} significantly outperforms the recent specialized methods for tail nodes~\citep{liu2021tail,zheng2021cold,zhang2021graph,kang2022rawlsgcn}, while not requiring any modification to GNN architectures nor losses. 
Overall, our work shows that even a simple training strategy can yield a surprisingly large improvement in the predictive performance of GNNs, pointing to a promising direction to investigate effective training strategies for GNNs, beyond architectures and losses.

\section{General Setup and Conventional Training}
\label{sec:setup}
Here we introduce a general task setup for machine learning on graphs and review the conventional training strategy of GNNs. 

\subsection{General Setup}
\label{sec:general_setup}
We are given a graph $G = (V, E)$, with a set of nodes $V$ and edges $E$, possibly associated with some features.
A GNN $F_{\theta}$, parameterized by $\theta$, takes the graph $G$ as input and makes prediction $\widehat{Y}$ for the task of interest. The loss function $L$ measures the discrepancy between the GNN's prediction $\widehat{Y}$ and the target supervision $Y$. 
When input node features are available, GNN $F_\theta$ can make not only transductive predictions, \ie, prediction over existing nodes $V$, but also inductive predictions~\citep{hamilton2017inductive}, \ie, prediction over new nodes $V_{\rm new}$ that are not yet present in $V$.
This is a general task setup that covers many representative predictive tasks over graphs as special cases:

\subsubsection{Semi-supervised node classification~\citep{kipf2016semi}}
% The task is to predict class labels of unlabeled nodes given a small set of labeled nodes in a graph, which can be formalized as follows.

\begin{itemize}[noitemsep,nolistsep,leftmargin=*]
    \item \textbf{Graph $G$}: A graph with input node features.
    \item \textbf{Supervison $Y$}: Class labels of labeled nodes $V_{\rm labeled} \subset V$. 
    \item \textbf{GNN $F_\theta$}: A model that takes $G$ as input and predicts class probabilities over $V$.
    \item \textbf{Prediction $\widehat{Y}$}: The GNN's prediction over $V_{\rm labeled}$.
    \item \textbf{Loss $L$}: Cross-entropy loss.
\end{itemize}

% Since input node features are available, the GNN $F_\theta$ can make not only transductive predictions, \ie, prediction over $V_{\rm unlabeled} \equiv V \setminus V_{\rm labeled}$, but also inductive predictions~\citep{hamilton2017inductive}, \ie, prediction over new nodes $V_{\rm new}$ that are not yet present in $V$.

\subsubsection{Link prediction~\citep{kipf2016variational}}
% The task is to predict new links in a graph given existing links.
% We consider the node-centric formulation~\citep{you2021identity}: given a source node, predict target nodes that the source node is linked to.

\begin{itemize}[noitemsep,nolistsep,leftmargin=*]
    \item \textbf{Graph $G$}: A graph with input node features.
    \item \textbf{Supervison $Y$}: Whether node $s\in V$ is linked to node $t \in V$ in $G$ (positive) or not (negative).
    \item \textbf{GNN $F_\theta$}: A model that takes $G$ as input and predicts the score for a pair of nodes $(s, t) \in V \times V$. Specifically, the model generates embedding $\bm{z}_v$ for each node in $v\in V$ and uses an MLP over the Hadamard product between $\bm{z}_s$ and $\bm{z}_t$ to predict the score for the pair $(s, t)$~\citep{grover2016node2vec}.
    \item \textbf{Prediction $\widehat{Y}$}: The GNN's predicted scores over $V \times V$.
    \item \textbf{Loss $L$}: The Bayesian Personalized Ranking (BPR) loss~\citep{rendle2012bpr}, which encourages the predicted score for the positive pair $(s, t_{\rm pos})$ to be higher than that for the negative pair $(s, t_{\rm neg})$ for each source node $s \in V$.
\end{itemize}

% As input node features are available, the GNN $F_{\theta}$ can naturally make inductive link prediction by generating node embeddings on a new graph with new nodes and edges.

\subsubsection{Recommender systems~\citep{wang2019neural}}
A recommender system is link prediction between user nodes $V_{\rm user}$ and item nodes $V_{\rm item}$.
% The task is essentially link prediction, \ie, given a user node $u \in V_{\rm user}$, predict a set of item nodes that $u$ is likely to interact with.
% In recommender systems, the most successful paradigm is collaborative filtering~\citep{schafer2007collaborative}, where shallow embeddings (learnable embeddings for each node) instead of input node features are used to achieve state-of-the-art performance~\citep{wang2019neural,he2020lightgcn}. Furthermore, because input node features are not available in many public recommender system datasets, we focus on the featureless setting.

\begin{itemize}[noitemsep,nolistsep,leftmargin=*]
    \item \textbf{Graph $G$}: User-item bipartite graph \emph{without} input node features.\footnote{We consider the feature-less setting because input node features are not available in many public recommender system datasets, and most existing works rely solely on edge connectivity to predict links.}
    \item \textbf{Supervison $Y$}: Whether a user node $u$ has interacted with an item node $v$ in $G$ (positive) or not (negative).
    \item \textbf{GNN $F_\theta$}: A model that takes $G$ as input and predicts the score for a pair of nodes $(u, v) \in V_{\rm user} \times V_{\rm item}$.
    Following~\citet{wang2019neural}, GNN parameter $\theta$ contains the input shallow embeddings in addition to the original message passing GNN parameter. To produce the score for the pair of nodes $(u, v)$, we generate the user and item embeddings, $\bm{z}_u$ and $\bm{z}_v$, and take the inner product $\bm{z}_u^{\top} \bm{z}_v$ to compute the score~\citep{wang2019neural}.
    \item \textbf{Prediction $\widehat{Y}$}: The GNN's predicted scores over $V_{\rm user} \times V_{\rm item}$.
    \item \textbf{Loss $L$}: The BPR loss~\citep{rendle2012bpr}.
\end{itemize}

\subsection{Conventional GNN Training}
A conventional way to train a GNN~\citep{kipf2016semi} is to minimize the loss $L(\widehat{Y}, Y)$ via gradient descent, as shown in L2--5 of Algorithm\ref{alg:method}. Extension to mini-batch training~\citep{hamilton2017inductive,zeng2019graphsaint} is straightforward by sampling subgraph $G$ in each parameter update.

\begin{algorithm}[t]\small
	\caption{\methodname{}. Compared to the conventional training of a GNN (L2--5), \methodname{} introduces a two-stage training process and only adds two components (L8 and L12) that are straightforward to implement. Each parameter update of \methodname{} is as efficient as the conventional GNN training.}
	\label{alg:method}
	{\bfseries Given:} GNN $F_{\theta}$, graph $G$, loss $L$, supervision $Y$, DropEdge ratio $\alpha$. 
	\begin{algorithmic}[1]
	    \STATE \textbf{\# First stage: Conventional training to obtain a base GNN.}
		\WHILE{$\theta$ not converged}
		\STATE Make prediction $\widehat{Y} = F_{\theta}(G)$
		\STATE Compute loss $L(\widehat{Y}, Y)$, compute gradient $\nabla_{\theta} L$, and update parameter $\theta$.
		\ENDWHILE
        \STATE \textbf{\# Set up for the second stage.}
	    \IF{task is semi-supervised node classification}
		\STATE \textcolor{red}{Use $F_{\theta}$ to predict pseudo-labels on non-isolated, unlabeled nodes. Add the pseudo-labels into $Y$.}
		\ENDIF
    \STATE \textbf{\# Second stage: Further training the base GNN with increased tail supervision.}
		\WHILE{$\theta$ not converged}
		\STATE \textcolor{red}{Synthesize tail nodes, \ie, randomly drop $\alpha$ of edges: $G \xrightarrow{\rm DropEdge} \widetilde{G}$.}
		\STATE Make prediction $\widehat{Y} = F_{\theta}(\widetilde{G})$.
		\STATE Compute loss $L(\widehat{Y}, Y)$, compute gradient $\nabla_{\theta} L$, and update parameter $\theta$.
		\ENDWHILE
	\end{algorithmic}
\end{algorithm}

\xhdr{Issue with Conventional Training}
Conventional training implicitly assumes GNNs can learn over all nodes equally well. In practice, some nodes, such as low-degree tail nodes, are more difficult for GNNs to learn due to the scarce neighborhood information. As a result, GNNs trained with conventional training often give poor predictive performance on the difficult tail nodes~\citep{liu2021tail}.

\section{\methodname{}: An Improved GNN Training}
\label{sec:tuneup}
To resolve the issue, we present \methodname{}, a simple curriculum learning strategy, to improve GNN performance, especially on the difficult tail nodes.
At a high level, \methodname{} first trains a GNN to perform well on the relatively easy head nodes. Then, it further trains the GNN to also perform well on the more difficult tail nodes.

Specifically, in the first stage (L2--5 in  Algorithm~\ref{alg:method}), \methodname{} uses conventional GNN training to obtain a strong base GNN model.
The base GNN model tends to perform well on head nodes, but poorly on tail nodes. To remedy this issue, in the second training stage, \methodname{} futher trains the base GNN on synthetic tail nodes (L7--L15 in Algorithm~\ref{alg:method}). \methodname{} synthesizes supervised tail node data in two steps, detailed next: (1) synthesizing additional tail node inputs, and (2) adding target supervision on the synthetic tail nodes.

\subsubsection{Synthesizing tail node inputs}
In many real-world graphs, head nodes start off as tail nodes, \eg, well-cited paper nodes are not cited at the beginning in a paper citation network, and warm users (users with many item interactions) start off as cold-start users in recommender systems.
Hence, our key idea is to synthesize tail nodes by systematically removing edges from head nodes.
There can be different ways to remove edges. In this work, we simply adopt DropEdge~\citep{rong2019dropedge} to instantiate our idea.
DropEdge drops a certain portion (given by hyperparameter $\alpha$) of edges randomly from the original graph $G$ to obtain $\widetilde{G}$ (L12 in Algorithm~\ref{alg:method}).
The resulting $\widetilde{G}$ contains more nodes with low degrees, \ie, tail nodes, than the original graph $G$. Hence, the GNN sees more (synthetic) tail nodes as input during training.

\subsubsection{Adding supervision on the synthetic tail nodes}
After synthesizing the tail node inputs, \methodname{} then adds target supervision (\eg, class labels for node classification, edges for link prediction) on them so that a supervised loss can be computed.
Our key idea is to reuse the target labels on the original head nodes for the synthesized tail nodes.
The rationale is that many prediction tasks involve target labels that are inherent node properties that do not change with node degree. For example, additional citations will not change a paper's subject area, and additional purchases will not change a product's category. 
% Even user preferences in recommender systems may not change too drastically as more user interactions are observed. 

% For these types of prediction tasks, it is effective to directly reuse the target supervision, as we will empirically demonstrate in Section~\ref{sec:experiments}.
% \footnote{Note that it may not always be appropriate to reuse the target supervision on head nodes.
% For instance, when predicting the influential nodes in social networks, the target label would correlate with the number of edges attached to the nodes. For such a case, we may need to ``correct'' the target label before reusing it for synthetic tail nodes, \eg, by scaling down the target label after removing edges.
% It is a fruitful future work to investigate how to effectively ``correct'' and ``reuse'' target labels on head nodes for different tasks.}

Concretely, for link prediction tasks, \methodname{} directly reuses the original edges $E$ in $G$ (before dropping) for the target supervision on the synthetic tail nodes. 
% To describe the effectiveness of this approach, suppose we have a node $v$ with six neighbors in the original training graph $G$. After dropping $\alpha=0.5$ of edges in L12 of Algorithm~\ref{alg:method}, this node becomes a synthetic tail node $\widetilde{v}$ with three neighbors in $\widetilde{G}$. Nevertheless, in the loss computation in L14, \methodname{} still reuses \emph{all original six edges} from $v$ in $G$ as target supervision on this synthetic node $\widetilde{v}$. Therefore, this synthetic tail node $\widetilde{v}$ has twice as much edge supervision as any real degree-three tail node in the original graph $G$. As a result, GNNs can see abundant synthetic supervised tail data in the second training stage, which in turn increases their predictive performance on the real tail nodes.
For semi-supervised node classification, \methodname{} can similarly reuse the target labels of labeled nodes $V_{\rm labeled}$ in $G$ as the labels for synthetic tail nodes in $\widetilde{G}$. 
% Specifically, for a labeled node $v \in V_{\rm labeled}$ with ground-truth class label $y_v$, \methodname{} can reuse $y_v$ for the corresponding synthetic tail node $\widetilde{v}$ in $\widetilde{G}.$
A critical challenge here is that the number of labeled nodes $V_{\rm labeled}$ is often small in the semi-supervised setting, \eg, 1\%--5\% of all nodes $V$, limiting the amount of target label supervision \methodname{} can reuse. 

To resolve this issue, \methodname{} applies the base GNN (obtained in the first training stage) over $G$ to predict pseudo-labels~\citep{lee2013pseudo} over non-isolated nodes in $V_{\rm unlabeled}\equiv V \setminus V_{\rm labeled}$.\footnote{Note that the pseudo-labels do not need to be ones directly predicted by the base GNN. For example, one can apply C\&S post-processing~\citep{huang2020combining} to improve the quality of the pseudo-labels, which we leave for future work.}
\methodname{} then includes the pseudo-labels as supervision $Y$ in the second stage (L8 in Algorithm~\ref{alg:method}). This significantly increases the size of the supervision $Y$, \eg, by a factor of $\approx$100 if only 1\% of nodes are labeled. 
While the pseudo-labels can be noisy, they are ``best guesses'' made by the base GNN in the sense that they are predicted using \emph{full graph information} $G$ as input.
In the second stage, \methodname{} trains the GNN to maintain its ``best guesses'' \emph{given sparser graph $\widetilde{G}$ as input}, which encourages the GNN to perform well on nodes whose neighbors are actually scarce. 
Note that this strategy is fundamentally different from the classical pseudo-labeling method~\citep{lee2013pseudo} that trains a model \emph{without} sparsifying the input graph.
In the following sections, we will see this both theoretically and empirically.

% \xhdr{Comparison to classical pseudo-labeling} It is worth emphasizing that \methodname{} is fundamentally different from the classical pseudo-label method for semi-supervised classification~\citep{lee2013pseudo}, in which the original full data/graph is used as input to predict the pseudo-labels.
% Unlike the conventional pseudo-label technique, \methodname{} \emph{sparsifies} the input graph $G$ when training GNNs, allowing the resulting GNNs to make more accurate predictions on \emph{real} tail nodes in the \emph{original graph} $G$. In Section~\ref{sec:experiments}, we empirically validate that sparsification is the key to performance improvement, and the conventional pseudo-label method alone gives marginal improvement.

\section{Theoretical Analysis}
\label{sec:theory}
To theoretically understand \methodname{} with clean insights, we consider node classification with binary labels for a part of a graph with two extreme groups of nodes: ones with full degrees and ones with zero degrees. Considering this part of a graph  as an example, we mathematically show how \methodname{} uses the nodes with high degrees to improve the generalization for the nodes with low degrees via the curriculum-based training strategy.

We analyze the generalization gap between  the test errors of nodes with low degrees and the training errors of nodes with high degrees. This type of generalization is non-standard and  does not necessarily happen unless we take advantage of some additional mechanisms such as dropping edges. Define $d$ to be  the dimensionality of the feature space of the last hidden layer of a GNN. Denote by $m$ the size of a set of labeled nodes used for training. Let  $Q$ be the average training loss at the end of the first stage of \methodname{} curriculum learning.

We prove a theorem (Theorem \ref{thm:1}), which  shows the following three statements: 
\begin{enumerate}[label=(\roman*),leftmargin=0.6cm]
\item 
First, consider \methodname{} without pseudo labeling (denoted by $M_1$). $M_1$ helps reduce the test errors of nodes with low degrees via utilizing the nodes with high degrees by dropping edges: \ie,  the  generalization bound in our theorem decreases towards zero at the rate of $\sqrt\frac{d}{m}$. 
\item 
The full \methodname{} (denoted by $M_2$) further reduces the test errors of nodes with low degrees by utilizing pseudo-labels: \ie,  the rate of$\sqrt\frac{d}{m}$  is replaced by the rate of $\sqrt\frac{1}{m}+Q$, where typically  $Q=0$  as $Q$ is explicitly minimized as a training objective. Thus, curriculum-based  training with pseudo-labels can remove the factor $d$.
\item 
\methodname{} without DropEdge (denoted by $M_3$), \ie, the classical pseudo-labeling method, degrades the test errors of nodes with low degrees by incurring additional error term $\tau>0$ that measures the difference between the losses with and without edges. This is consistent with the above intuition that generalizing from high-degree training nodes to the low-degree test nodes requires some relationship between ones with and without edges.          
\end{enumerate}  

% Denote by  $M_1$ and   $M_2$  the proposed method without and with the  pseudo-labels, respectively. We use $M_3$ to denote the proposed method with pseudo-labels but without dropping edges. 
For each method $M \in \{M_i\}_{i=1}^3$, we define $\Delta(M)$ to be the generalization gap between  the test errors of nodes with low degrees and the training errors of nodes with high degrees. 

\begin{theorem} \label{thm:1}
 For any $\delta>0$, with probability at least  $1-\delta$, the following holds for all $M \in \{M_1,M_2,M_{3}\}$: 
 
\begin{align*}
\nonumber \Delta(M) & \le\sqrt{\frac{ \one\{M=M_1\}8( d-1)\ln(\frac{16 e m}{\delta})+8\ln(\frac{16 e m}{\delta})}{m}}
\\ & \quad +\one\{M\neq M_1\} Q+\one\{M= M_{3}\}\tau+G,
\end{align*} 
where  $G \rightarrow 0$ as the graph size approaches infinity.
\end{theorem}
\begin{proof}
A more detailed version of Theorem \ref{thm:1} is presented along with the complete proof in Appendix. 
\end{proof}

\section{Related Work}
\label{sec:related}
\subsection{Methods for Tail Nodes}
Recently, a surge of methods have been developed for improving the predictive performance of GNNs on tail nodes~\citep{liu2021tail,zheng2021cold,kang2022rawlsgcn,zhang2021graph}.
These methods augment GNNs with complicated tail-node-specific architectural components and losses, while \methodname{} focuses on the training strategy that does not require any architectural nor loss modification.  

% \subsection{Curriculum learning for GNNs}
% A few works have explored curriculum learning for GNNs.
% \citet{wang2021curgraph} developed a curriculum learning approach for graph classification, while our work focuses on node- and edge-level prediction tasks. \citet{ying2018graph} presented a curriculum learning for negative sampling in link prediction, and \citet{li2022graph} developed a curriculum learning for tackling imbalanced class labels in node classification. \methodname{} is complementary to both of these approaches while being more broadly applicable to any node- and edge-level prediction tasks.

\subsection{Data augmentation for GNNs}
The second stage of \methodname{} is data augmentation over graphs, on which there has been abundant work~\citep{zhao2021data,feng2020graph,verma2021graphmix,kong2020flag,liu2022local,ding2022data}.
The most relevant one is DropEdge~\citep{rong2019dropedge}, which was originally developed to overcome the over-smoothing issue of GNNs~\citep{li2018deeper} specific to semi-supervised node classification.
Our work has a different motivation and expended scope: We use DropEdge to synthesize tail node inputs and consider a wider range of prediction tasks. Our theoretical analysis also differs and focuses on generalization on tail nodes.
Methodologically, \methodname{} additionally employs curriculum learning and pseudo-labels, both of which are crucial in improving GNN performance over the vanilla DropEdge.

\section{Experiments}
\label{sec:experiments}
We evaluate the broad applicability of \methodname{} by considering five GNN models and testing them on the three prediction tasks (semi-supervised node classification, link prediction, and recommender systems) with three predictive settings: transductive, inductive, and cold-start inductive predictions. 

\subsection{Experimental Settings}
We evaluate \methodname{} on realistic tail node scenarios in both transductive (\ie, naturally occurring tail nodes in scale-free networks) and inductive (\ie, newly arrived cold-start nodes) settings.
Conventional experimental setups~\citep{hu2020open,wang2019neural} are not suitable for evaluating \methodname{} as they fail to provide either (1) transductive prediction settings with tail nodes,\footnote{Recommender system benchmarks are processed with the 10-core algorithm to eliminate cold-start users and items~\citep{wang2019neural}.} or (2) inductive cold-start prediction settings.
Therefore, we split the original realistic graph datasets~\citep{hu2020open,wang2019neural} to simulate both (1) and (2) in a realistic manner.
% \wh{TODO: Discuss why OGB's link prediction eval is not appropriate.}
% For each experiment, we selected the hyperparameters based on the transductive validation set and used the selected model to make predictions on both the transductive and (cold-start) inductive test sets.
Below, we describe the split for each task type.
The dataset statistics are summarized in Table~\ref{tab:dataset_stats}.

\begin{table}[t]
\centering
    \captionof{table}{Statistics of nodes used for the transductive evaluation. For link prediction and recommender system graphs (user-item bipartite graphs), we only evaluate on nodes/users with at least one edge in the validation set. See Appendix for the description of the datasets.}
    \label{tab:dataset_stats}
    \resizebox{\linewidth}{!}{
    \begin{tabular}{llrrr}
      \toprule
       \textbf{Task} & \textbf{Dataset} & \textbf{\#Nodes} & \textbf{Avg deg.} & \textbf{Feat. dim}\\
      \midrule
        Node & arxiv & 143,941 & 12.93 & 128 \\
        classification  & products & 2,277,597 & 48.01 & 100 \\
          \midrule
        Link & flickr  & 82,981 & 4.81 & 500 \\
        % & ppi  & 15,390 & 17.30 & 1,280 \\
         prediction & arxiv  & 141,917 & 7.20 & 128 \\
         \midrule
        \mr{2}{Recsys} & gowalla & 29,858 & 3.44 & -- \\
        % & yelp2018 & 31,668 & 4.93 & --  \\
        & amazon-book & 52,643 & 5.67 & -- \\ 
      \bottomrule
    \end{tabular}
    }
\end{table}

\subsection{Semi-supervised node classification} Given all nodes in the original dataset, we randomly selected 95\% of the nodes and used their induced subgraph as the graph $G = (V, E)$ to train GNNs. The remaining 5\% of nodes, $V_{\rm new}$, is used for inductive test prediction.
Within $V$, 10\% and 2\% of the nodes are used as labeled nodes $V_{\rm labeled}$ for arxiv and products, respectively. Half of $V_{\rm labeled}$ is used to compute the supervised training loss, and the other half is used as the transductive validation set for selecting hyperparameters.
We used classification accuracy for the evaluation metric.
For the transductive performance, we report the test accuracy on the unlabeled test nodes $V_{\rm unlabeled} \equiv V \setminus V_{\rm labeled}$, while for the inductive performance, we report the test accuracy on $V_{\rm new}$.
For the inductive test prediction, we also consider the \emph{cold-start} scenario, where certain portions (30\%, 60\%, and 90\%) of edges are randomly removed from the new nodes.

\begin{figure*}
\centering
    \captionof{table}{The improvement with \methodname{} over the base GNNs for five diverse GNN model architectures. We used the same datasets as Figure~\ref{fig:base_comp}. For semi-supervised node classification, ``inductive (cold)'' randomly removed 90\% of edges from the new nodes, while for link prediction, 60\% were removed.
    $^\dagger$For semi-supervised node classification, GAT gave Out-Of-Memory (OOM) on the products dataset, so we report the performance on arxiv instead.}
   \label{tab:gnn_model_table}
    \resizebox{\linewidth}{!}{
\begin{tabular}{l|l|l|ccccc}
      \toprule
          Task & Config & Setting & SAGE & GCN & SAGE-max & SAGE-sum & GAT \\ \midrule
           & \mr{2}{Transductive} &  Base & 0.8409\std{0.0006} & 0.8432\std{0.0007} & 0.8132\std{0.0004} & 0.7611\std{0.0030} & OOM / 0.6862\std{0.0023}$^\dagger$  \\
          Semi-sup & & \methodname{} & \textbf{0.8552\std{0.0003}} & \textbf{0.8523\std{0.0007}} & \textbf{0.8373\std{0.0008}} & \textbf{0.7612\std{0.0030}} & OOM / \textbf{0.6973\std{0.0015}}$^\dagger$  \\ 
          node & \mr{2}{Inductive} & Base & 0.8425\std{0.0006} & 0.8447\std{0.0008} & 0.8129\std{0.0012} & 0.7610\std{0.0024} & OOM / 0.6800\std{0.0024}$^\dagger$  \\
          classification & & \methodname{} & \textbf{0.8562\std{0.0005}} & \textbf{0.8536\std{0.0006}} & \textbf{0.8374\std{0.0013}} & \textbf{0.7616\std{0.0029}} & OOM / \textbf{0.6930\std{0.0013}} \\ 
          (products) & \mr{2}{Inductive (cold)} & Base & 0.7227\std{0.0011} & 0.7461\std{0.0033} & 0.6907\std{0.0007} & 0.5331\std{0.0078} & OOM / 0.5405\std{0.0034}$^\dagger$  \\  
          & & \methodname{} & \textbf{0.8054\std{0.0011}} & \textbf{0.7924\std{0.0050}} & \textbf{0.7868\std{0.0012}} & \textbf{0.5366\std{0.0129}} & OOM / \textbf{0.5966\std{0.0053}}$^\dagger$  \\  \midrule
             & \mr{2}{Transductive} & Base & 0.1371\std{0.0028} & 0.2242\std{0.0005} & 0.1697\std{0.0024} & 0.0761\std{0.0010} & 0.2363\std{0.0016} \\
          Link & & \methodname{} & \textbf{0.2161\std{0.0020}} & \textbf{0.2527\std{0.0017}} & \textbf{0.2489\std{0.0027}} & \textbf{0.1209\std{0.0108}} & \textbf{0.2648\std{0.0033}} \\ 
          prediction & \mr{2}{Inductive} & Base & 0.1227\std{0.0042} & 0.2052\std{0.0005} & 0.1484\std{0.0012} & 0.0684\std{0.0015} & 0.2020\std{0.0034} \\ 
          (arxiv) & & \methodname{} & \textbf{0.1807\std{0.0044}} & \textbf{0.2239\std{0.0027}} & \textbf{0.2141\std{0.0020}} & \textbf{0.1060\std{0.0083}} & \textbf{0.2335\std{0.0033}}  \\ 
          & \mr{2}{Inductive (cold)} & Base & 0.0688\std{0.0020} & 0.1185\std{0.0011} & 0.0992\std{0.0041} & 0.0508\std{0.0012} & 0.1273\std{0.0024} \\ 
          & & \methodname{} & \textbf{0.1241\std{0.0025}} & \textbf{0.1428\std{0.0021}} & \textbf{0.1559\std{0.0030}} & \textbf{0.0785\std{0.0046}} & \textbf{0.1580\std{0.0033}}  \\ \midrule
          
          Recsys & \mr{2}{Transductive} & Base & 0.0847\std{0.0006} & 0.0901\std{0.0004} & 0.0858\std{0.0006} & 0.0761\std{0.0010} & 0.0803\std{0.0005} \\
            (gowalla)            &  & \methodname{} & \textbf{0.1025\std{0.0018}} & \textbf{0.1094\std{0.0007}} & \textbf{0.1055\std{0.0025}} & \textbf{0.1028\std{0.0012}} & \textbf{0.0822\std{0.0006}} \\
      \bottomrule
    \end{tabular}}
\end{figure*}

\subsection{Link prediction}
We follow the same protocol as above to obtain transductive nodes $V$ and inductive nodes $V_{\rm new}$.
For transductive evaluation, we randomly split the edges $E$ into training/validation/test sets with the ratio of 50\%/20\%/30\%~\citep{zhang2018link,you2021identity}.
For inductive evaluation, we randomly split the edges from the new nodes $V_{\rm new}$ into training/test edges with a ratio of 50\%/50\%. During the inductive inference time, the training edges are used as input to GNNs, with the GNN parameters fixed.

For the evaluation metric, we use the recall@50~\citep{wang2019neural}, where the positive target nodes are scored among all negative nodes.\footnote{Our evaluation protocol is more realistic~\citep{krichene2020sampled} than the OGB link prediction datasets that evaluate each positive edge among randomly selected edges~\citep{hu2020open} .}
We use the validation recall@50 averaged over $V$ to tune hyper-parameters.
For the transductive performance, we report the test recall@50 averaged over the $V$, while for inductive performance, we report the recall@50 averaged over $V_{\rm new}$. For the inductive setting, we also consider the cold-start scenario, as we have described in the semi-supervised node classification.

\subsection{Recommender systems}
For recommender systems, we noticed that widely-used benchmark datasets were heavily processed to eliminate all tail nodes, \eg, via the 10-core algorithm~\citep{wang2019neural}. As a result, the conventional 80\%/10\%/10\% train/validation/test split gives the median training interactions per user of 17 and 27 for gowalla and amazon-book, respectively, which do not reflect the realistic use case that involves cold-start users and items~\citep{lika2014facing}.
To reflect the realistic use case, we use a smaller training edge ratio on top of the existing benchmark datasets. Specifically, we randomly split the edges in the original graph into training/validation/test edges with a 10\%/5\%/85\% ratio.
We use the same evaluation metric and protocol as link prediction, except that we do not consider the inductive setting in recommender systems due to the absence of input node features.
% We use recall@50 as evaluation metric~\citep{wang2019neural,he2020lightgcn}.
% We use the validation recall@50 
% For the transductive performance, we report the test recall@50 computed over the test edges. We did not consider the inductive setting for recommender systems, as input node features are unavailable.

\begin{table*}[t]
    \centering
        \caption{Semi-supervised node classification performance with GraphSAGE as the backbone architecture. The metric is classification accuracy. For the ``Inductive (cold)'', 90\% of edges are randomly removed from the new nodes. For the results with other edge removal ratios, refer to Table~\ref{tab:cold_nodecls_sage_table} in Appendix.
        Refer to Table \ref{tab:nodecls_gcn_table} in Appendix for the performance with GCN, where a similar trend is observed. } 
    \label{tab:nodecls_sage_table}
    \renewcommand{\arraystretch}{1.0}
    \setlength{\tabcolsep}{5pt}
     \resizebox{\linewidth}{!}{
    % \adjustbox{max width=\textwidth}{%
    \begin{tabular}{l|ccc|ccc}
      \toprule
     \multirow{2}{*}{\textbf{Method}} & \multicolumn{3}{c|}{\textbf{arxiv}} & \multicolumn{3}{c}{\textbf{products}} \\  \cmidrule{2-7}
        & \textbf{Transductive}  & \textbf{Inductive} & \textbf{Inductive (cold)} & \textbf{Transductive} & \textbf{Inductive} & \textbf{Inductive (cold)} \\
      \midrule
Base & 0.6738\std{0.0007} & 0.6686\std{0.0005} & 0.4752\std{0.0061} & 0.8409\std{0.0006} & 0.8425\std{0.0006} & 0.7227\std{0.0011} \\ 
DropEdge & 0.6756\std{0.0013} & 0.6690\std{0.0032} & 0.5449\std{0.0059} & 0.8464\std{0.0006} & 0.8472\std{0.0006} & 0.7709\std{0.0014} \\ 
LocalAug & 0.6830\std{0.0007} & \textbf{0.6768\std{0.0010}} & 0.4981\std{0.0018} & 0.8445\std{0.0004} & 0.8461\std{0.0004} & 0.7261\std{0.0008} \\ 
ColdBrew & 0.6726\std{0.0007} & 0.6487\std{0.0007} & 0.5082\std{0.0018} & 0.8374\std{0.0005} & 0.8382\std{0.0004} & 0.7395\std{0.0019} \\ 
GraphLessNN & 0.6076\std{0.0009} & 0.5456\std{0.0008} & 0.5456\std{0.0008} & 0.6678\std{0.0007} & 0.6648\std{0.0009} & 0.6648\std{0.0009} \\ 
Tail-GNN & 0.6614\std{0.0013} & 0.6548\std{0.0011} & 0.5388\std{0.0031} & OOM & OOM & OOM \\ 
\midrule
\methodname{} w/o curriculum & 0.6753\std{0.0014} & 0.6682\std{0.0020} & 0.5472\std{0.0119} & 0.8458\std{0.0005} & 0.8467\std{0.0008} & 0.7569\std{0.0015} \\ 
\methodname{} w/o pseudo-labels & 0.6745\std{0.0007} & 0.6672\std{0.0019} & 0.5332\std{0.0077} & 0.8462\std{0.0005} & 0.8472\std{0.0008} & 0.7631\std{0.0055} \\ 
\methodname{} w/o syn-tails & 0.6787\std{0.0008} & \textbf{0.6760\std{0.0006}} & 0.4899\std{0.0047} & 0.8436\std{0.0003} & 0.8451\std{0.0003} & 0.7258\std{0.0011} \\ 
\textbf{\methodname{} (ours)} & \textbf{0.6872\std{0.0008}} & \textbf{0.6779\std{0.0026}} & \textbf{0.5996\std{0.0012}} & \textbf{0.8552\std{0.0003}} & \textbf{0.8562\std{0.0005}} & \textbf{0.8054\std{0.0011}} \\ 
\midrule
Rel. gain over base & +2.0\% & +1.4\% & +26.2\% & +1.7\% & +1.6\% & +11.4\% \\ 
    \bottomrule
    \end{tabular}
    % }
    }
 
\end{table*}

\subsection{Baselines and Ablations}
\label{subsec:baseline}
We compared \methodname{} against the following baselines.
\begin{itemize}[noitemsep,nolistsep,leftmargin=*]
    \item \textbf{Base}: Trains a GNN with the conventional strategy, \ie, L2--5 of Algorithm~\ref{alg:method}. Note that our pseudo-labels are produced by this base GNN.
    \item \textbf{DropEdge}~\citep{wang2019neural}: Randomly drops edges during training, \ie, L11--15 of Algorithm~\ref{alg:method}.
    \item \textbf{Local augmentation (LocalAug)}~\citep{liu2022local}: Uses a conditional generative model to generate neighboring node features and use them as additional input to a GNN.
    \item \textbf{ColdBrew}~\citep{zheng2021cold}: Distills head node embeddings computed by the base GNN into an MLP. Uses the resulting MLP to obtain higher-quality tail node embeddings.
    \item \textbf{GraphLessNN}~\citep{zhang2021graph}: Distills the pseudo-labels predicted by the base GNN into an MLP. Uses the resulting MLP to make prediction.
    \item \textbf{Tail-GNN}~\citep{liu2021tail}: Adds a tail-node-specific component inside the original GNN.
    \item \textbf{RAWLS-GCN}~\citep{kang2022rawlsgcn}: Modifies the GCN's adjacency matrix to be doubly stochastic (\ie, all rows and columns sum to 1). 
\end{itemize}
Note that GraphLessNN is only applicable for node classification. LocalAug and ColdBrew require input node features to be available; hence, they are not applicable to recommender systems. RAWLS-GCN is only applicable to the GCN architecture.

In addition to the existing baselines, we consider the following three direct ablations of \methodname{}.
\begin{itemize}[noitemsep,nolistsep,leftmargin=*]
    \item \textbf{\methodname{} w/o curriculum}: Interleaves the first stage prediction (L3 in Algorithm~\ref{alg:method}) and the second stage prediction (L12--13 in Algorithm~\ref{alg:method}) in every parameter update. It is close to \methodname{} except that it does not follow the two-stage curriculum learning strategy. 
    \item \textbf{\methodname{} w/o syn-tails}: No L12 in Algorithm~\ref{alg:method}.
    \item \textbf{\methodname{} w/o pseudo-labels}: No L8 in Algorithm~\ref{alg:method}.
\end{itemize}
Another possible ablation, \methodname{} w/o the first stage training (\ie, only performing the second stage training of L2--5 in Algorithm~\ref{alg:method}), is covered as DropEdge in our experiments.

\begin{table*}[t]
    \centering
        \caption{Link prediction performance with GraphSAGE as the backbone architecture. The metric is recall@50. For the ``Inductive (cold)'', 60\% of edges are randomly removed from the new nodes. For other edge removal ratios, refer to Table \ref{tab:cold_linkpred_sage_table} in Appendix, where \methodname{} consistently outperforms the baselines. Refer to Table \ref{tab:linkpred_gcn_table} in Appendix for the performance with GCN, where we see a similar trend. }
         
    \label{tab:linkpred_sage_table}
    \renewcommand{\arraystretch}{1.0}
    \setlength{\tabcolsep}{5pt}
    \resizebox{\linewidth}{!}{
    \begin{tabular}{l|ccc|ccc}
      \toprule
     \mr{2}{\textbf{Method}} & \mc{3}{c|}{\textbf{flickr}}  & \mc{3}{c}{\textbf{arxiv}} \\ \cmidrule{2-7}
         & \textbf{Transductive} & \textbf{Inductive} & \textbf{Inductive (cold)} & \textbf{Transductive} & \textbf{Inductive} & \textbf{Inductive (cold)} \\
        \midrule
Base & 0.1023\std{0.0019} & 0.1012\std{0.0034} & 0.0582\std{0.0014} & 0.1371\std{0.0028} & 0.1227\std{0.0042} & 0.0688\std{0.0020} \\ 
DropEdge & 0.1359\std{0.0020} & 0.1283\std{0.0015} & 0.0992\std{0.0008} & 0.2109\std{0.0049} & 0.1748\std{0.0039} & 0.1189\std{0.0046} \\ 
LocalAug & 0.1073\std{0.0030} & 0.1089\std{0.0028} & 0.0646\std{0.0059} & 0.1434\std{0.0048} & 0.1269\std{0.0044} & 0.0734\std{0.0036} \\ 
ColdBrew & 0.0716\std{0.0062} & 0.0700\std{0.0070} & 0.0369\std{0.0045} & 0.1242\std{0.0047} & 0.1103\std{0.0051} & 0.0640\std{0.0031} \\ 
Tail-GNN & 0.0790\std{0.0022} & 0.0712\std{0.0026} & 0.0657\std{0.0016} & 0.1007\std{0.0035} & 0.0847\std{0.0032} & 0.0586\std{0.0031} \\ 
\midrule
\methodname{} w/o curriculum & 0.1406\std{0.0005} & 0.1322\std{0.0010} & 0.1014\std{0.0018} & 0.2064\std{0.0050} & 0.1725\std{0.0058} & 0.1144\std{0.0041} \\ 
\methodname{} w/o syn-tails & 0.1015\std{0.0018} & 0.0997\std{0.0033} & 0.0583\std{0.0013} & 0.1412\std{0.0032} & 0.1259\std{0.0028} & 0.0728\std{0.0032} \\ 
\textbf{\methodname{} (ours)} & \textbf{0.1464\std{0.0033}} & \textbf{0.1384\std{0.0040}} & \textbf{0.1119\std{0.0069}} & \textbf{0.2161\std{0.0020}} & \textbf{0.1807\std{0.0044}} & \textbf{0.1241\std{0.0025}} \\ 
\midrule
Rel. gain over base & +43.2\% & +36.8\% & +92.2\% & +57.6\% & +47.4\% & +80.4\% \\ 
    \bottomrule
    \end{tabular}
    }

\end{table*}

\begin{figure}
\centering
    \captionof{table}{Transductive performance on the recommender systems datasets. The metric is recall@50. }%
   \label{tab:recsys_table}
    \resizebox{1.05\linewidth}{!}{
\begin{tabular}{l|cc|cc}
      \toprule
          \mr{2}{\textbf{Method}} & \mc{2}{c|}{\textbf{gowalla}}  &  \mc{2}{c}{\textbf{amazon-book}} \\ 
          \cmidrule{2-5}
          & SAGE & GCN  &  SAGE & GCN \\
          \midrule
 Base & 0.0847\std{0.0006} & 0.0901\std{0.0004} & 0.0545\std{0.0003} & 0.0527\std{0.0001} \\ 
 DropEdge & 0.0827\std{0.0002} & 0.0814\std{0.0004} & 0.0525\std{0.0011} & 0.0539\std{0.0005} \\ 
 RAWLS-GCN & -- & 0.0625\std{0.0005} & -- & 0.0469\std{0.0002} \\ 
 Tail-GNN & 0.0791\std{0.0005} & 0.0777\std{0.0011} & 0.0550\std{0.0002} & 0.0518\std{0.0005} \\ 
 \midrule
 \methodname{} w/o curriculum & 0.0834\std{0.0005} & 0.0857\std{0.0042} & 0.0537\std{0.0001} & 0.0525\std{0.0005} \\ 
 \methodname{} w/o syn-tails & 0.0847\std{0.0006} & 0.0904\std{0.0004} & 0.0546\std{0.0003} & 0.0530\std{0.0002} \\ 
 \textbf{\methodname{} (ours)} & \textbf{0.1025\std{0.0018}} & \textbf{0.1094\std{0.0007}}  & \textbf{0.0558\std{0.0007}} & \textbf{0.0618\std{0.0003}} \\ 
 \midrule
 Rel. gain over base & +21.1\% & +21.4\%  & +2.3\% & +17.3\% \\ 
      \bottomrule
    \end{tabular}}
\end{figure}

\subsection{GNN Model Architectures}

We mainly experimented with two classical yet strong GNN models: the mean-pooling variant of GraphSAGE (or SAGE for short)~\citep{hamilton2017inductive} and GCN~\citep{kipf2016semi}.
In Table~\ref{tab:gnn_model_table}, we additionally experimented with the max- and sum-pooling variants of GraphSAGE as well as the Graph Attention Network (GAT)~\citep{velivckovic2017graph}.
% to demonstrate the applicability of \methodname{} on diverse GNN architectures to improve their predictive performance.
In total, the experimented GNN architectures cover diverse aggregation schemes of mean, renormalized-mean~\citep{kipf2016semi}, max, sum, and attention, which are also building blocks of more recent GNN architectures~\citep{corso2020principal,shirunimp,you2020graph,wu2019simplifying,rossi2020sign,li2018deeper,you2020design}.

\subsection{Hyperparameters}
\label{subsec:hyper-parameters}
We used three-layer GNNs and the Adam optimizer~\citep{kingma2014adam} for all GNN models and datasets, which we found to perform well in our preliminary experiments. 
For all methods, we performed early stopping and selected hyperparameters based on the transductive validation performance. For the drop edge ratio $\alpha$, we selected it from [0.25, 0.5, 0.75] for all datasets and methods. We repeated all experiments with five training seeds to report the mean and standard deviation.
More details are described in Appendix \ref{app:hyperparams}.

\subsection{Results}
We first compare \methodname{} against the base GNNs trained with the conventional strategy.
Table~\ref{tab:gnn_model_table} summarizes the results across the three different prediction tasks, five diverse GNN architectures, and both transductive and inductive (cold-start) settings.
We see that \methodname{} improves the predictive performance of GNNs across the settings, indicating its general usefulness in training GNNs. One minor exception is the sum aggregation in semi-supervised node classification, but the sum aggregation is non-standard in semi-supervised classification anyway due to the poor model performance and inappropriate inductive bias~\citep{wu2019simplifying}. 

We also analyze the degree-specific predictive improvement and highlight the results in Figure~\ref{fig:base_comp}. 
The full results (the five GNN architectures times the six datasets) are available in Figures~\ref{fig:nodecls_comp}, \ref{fig:linkpred_comp}, and \ref{fig:recsys_comp} in Appendix.
We see that \methodname{} produces consistent improvement over the base GNNs across the node degrees. Not surprisingly, improvement is most significant on tail nodes. 

We then focus on the two representative GNNs (GraphSAGE and GCN) and provide extensive results in Tables \ref{tab:nodecls_sage_table}--\ref{tab:recsys_table}.
Overall,  \methodname{} establishes its superior performance over the existing strong baseline methods by outperforming the graph augmentation methods (DropEdge and LocalAug) as well as the specialized methods for tail nodes (ColdBrew, GraphLessNN, and Tail-GNN) \emph{across the three different tasks}. In particular, from the ``Inductive (cold)'' column, we see that \methodname{} gives superior performance than ColdBrew, GraphLessNN, and Tail-GNN on the cold-start tail nodes, despite its simplicity and not requiring any loss/architectural change.

Moreover, \methodname{} outperforms \methodname{} w/o curriculum, which highlights the importance of the two-stage curriculum learning strategy in \methodname{}. \methodname{} also outperforms \methodname{} w/o syn-tails and \methodname{} w/o pseudo-labels, which suggests that both of the ablated components are necessary, as predicted by our theory. \methodname{} is the only method that yielded consistent improvement over the base GNNs, indicating its broad applicability across the prediction tasks.
More detailed discussion can be found in Appendix.

\section{Conclusions}
\label{sec:conclusion}

We presented \methodname{}, a simple two-stage curriculum learning strategy for improving GNN performance, especially on tail nodes. 
\methodname{} is simple to implement, does not require any modification to loss or model architecture, and can be used with a wide range of GNN architectures.
Through extensive experiments, we demonstrated the effectiveness of \methodname{} in diverse settings, including five GNN architectures, three types of prediction tasks, and three settings (transductive, inductive, and cold-start).
Overall, our work suggests that even a simple training strategy can significantly improve the predictive performance of GNNs and complement parallel advances in model architectures and losses.

\section{Acknowledgments}
We thank Rajas Bansal for discussion. We also thank Camilo Ruiz and Qian Huang for providing  feedback on our manuscript. Our codebase is built using Pytorch~\citep{paszke2019pytorch} and Pytorch Geometric~\citep{fey2019fast}.
Weihua Hu is supported by Funai Overseas Scholarship and Masason Foundation Fellowship. 
We also gratefully acknowledge the support of
DARPA under Nos. HR00112190039 (TAMI), N660011924033 (MCS);
ARO under Nos. W911NF-16-1-0342 (MURI), W911NF-16-1-0171 (DURIP);
NSF under Nos. OAC-1835598 (CINES), OAC-1934578 (HDR), CCF-1918940 (Expeditions), 
NIH under No. 3U54HG010426-04S1 (HuBMAP),
Stanford Data Science Initiative, 
Wu Tsai Neurosciences Institute,
Amazon, Docomo, GSK, Hitachi, Intel, JPMorgan Chase, Juniper Networks, KDDI, NEC, and Toshiba.

The content is solely the responsibility of the authors and does not necessarily represent the official views of the funding entities.

\bibliography{reference}

\newpage

\appendix
\onecolumn
\section{Details of Datasets}
\label{app:datasets}

\subsection{Semi-supervised node classification}
We used the following two datasets:
\begin{itemize}[noitemsep,nolistsep,leftmargin=*]
    \item \textbf{arxiv}~\citep{hu2020open}: Given a paper citation network, the task is to predict the subject areas of the papers.
    Each paper has abstract words as its feature. 
    \item \textbf{products}~\citep{hu2020open}: Given a product co-purchasing network, the task is to predict the categories of the products. Each product has the product description as its feature. 
\end{itemize}

\subsection{Link prediction}
We used the following two datasets:
\begin{itemize}[noitemsep,nolistsep,leftmargin=*]
    \item \textbf{flickr}~\citep{zeng2019graphsaint}: Given an incomplete image-image common-property (\eg, same geographic location, same gallery, comments by the same user, etc.) network, the task is to predict the new common-property links between images. Each image has its description has its feature.
    % \item \textbf{ppi}~\citep{chandak2022building}: Given an incomplete protein-protein interaction network, the task is to predict new interactions. Each protein feature is generated with ESM protein language model~\citep{rives2021biological} applied to the protein sequence.
    \item \textbf{arxiv}~\citep{hu2020open}: Given an incomplete paper citation network, the task is to predict the additional citation links. Each paper has words in its abstract as its feature.
\end{itemize}
\vspace{0.5cm}

\subsection{Recommender systems}
We used the following two datasets:
\begin{itemize}[noitemsep,nolistsep,leftmargin=*]
    \item \textbf{gowalla}~\citep{liang2016modeling,wang2019neural}: Given a user-location check-in bipartite graph, the task is to predict new check-in of users.
    \item \textbf{amazon-book}~\citep{he2016ups,wang2019neural}: Given user-product reviews, the task is to predict new reviews by users.
\end{itemize}

\section{Details of Hyperparameters}
\label{app:hyperparams}

Here we present the details of hyperparameters we used in our experiments.

\subsection{Semi-supervised node classification} We used the hidden dimensionality of 256 and 64 for arxiv and products, respectively. We trained GNNs in a full-batch manner, and for products, we used the reduced dimensionality of 64 so that the entire graph fits into the limited GPU memory of 45GB. Mini-batch training is left for future work. 
We used 1500 epochs for both default training and fine-tuning. The learning rate is set to 0.001. 

\subsection{Link prediction}
We used the hidden dimensionality of 256 for all datasets. We added L2 regularization on the node embeddings and tuned its weight for each dataset and GNN architecture.
For both default training and fine-tuning, we used 1000 epochs and a learning rate of 0.0001.

\subsection{Recommender systems}
We used the shallow embedding dimensionality of 64 and the hidden embedding dimensionality of 256. Similar to link prediction, we added L2 regularization to the node embeddings and tuned its weight for each dataset and GNN architecture.
For default training, we trained the model for 2000 epochs with an initial learning rate of 0.001, which is multiplied by 0.1 at the 1000th and 1500th epoch. For fine-tuning, we used 500 epochs with a learning rate of 0.0001.

For training strategies without curriculum learning, we used the same configuration as the default training.

\section{Detailed Discussion on Experimental Results}
Below, we provide a detailed discussion of our experimental results.
\begin{itemize}[noitemsep,nolistsep,leftmargin=*]
    \item The last rows of Tables \ref{tab:nodecls_sage_table}, \ref{tab:linkpred_sage_table}, and \ref{tab:recsys_table} highlight the relative improvement of \methodname{} over the base GNNs. \methodname{} improves over the base GNNs across the transductive settings, giving up to 2.0\%, 57.6\%, and 21.1\% relative improvement in the semi-supervised node classification, link prediction, and recommender systems, respectively. Moreover, \methodname{} gave even larger improvements on the challenging cold-start inductive prediction setting, yielding up to 26.2\% and 92.2\% relative improvement on node classification and link prediction, respectively. 
    \item In Tables~\ref{tab:cold_nodecls_sage_table} and  \ref{tab:cold_linkpred_sage_table}, we show the results of the challenging cold-start inductive prediction with the three different edge removal ratios from the new nodes. From the last rows of the tables, we see that \methodname{} provides larger relative gains on larger edge removal ratios, demonstrating its high effectiveness on the highly cold-start prediction setting.
    \item On semi-supervised node classification (Tables \ref{tab:nodecls_sage_table} and \ref{tab:cold_nodecls_sage_table}), \methodname{} w/o syn-tails gave limited performance improvement, indicating that the conventional semi-supervised training with the pseudo-labels~\citep{lee2013pseudo} is not as effective. Moreover, \methodname{} significantly outperforms \methodname{} w/o pseudo-labels, especially in the inductive cold-start scenario, suggesting the benefit of pseudo-labels in increasing the supervised tail node data.
    \item On link prediction (Tables \ref{tab:linkpred_sage_table} and \ref{tab:cold_linkpred_sage_table}), DropEdge (\methodname{} without the first stage) already gave significant performance improvement over the base GNN. This implies the unrealized potential of DropEdge on this task, beyond mitigating oversmoothing in node classification~\citep{rong2019dropedge}.
    Nonetheless, \methodname{} still gave consistent improvement over DropEdge, suggesting the benefit of the two-stage training.
    \item On recommender systems (Table \ref{tab:recsys_table}), \methodname{} is the only method that produced significantly better performance than the base GNN. DropEdge and \methodname{} w/o curriculum performed worse than the base GNN. This may be because jointly learning the GNN and shallow embeddings is hard without the two-stage training.
\end{itemize}

\begin{figure*}[b]
\centering
\includegraphics[width=1\linewidth]{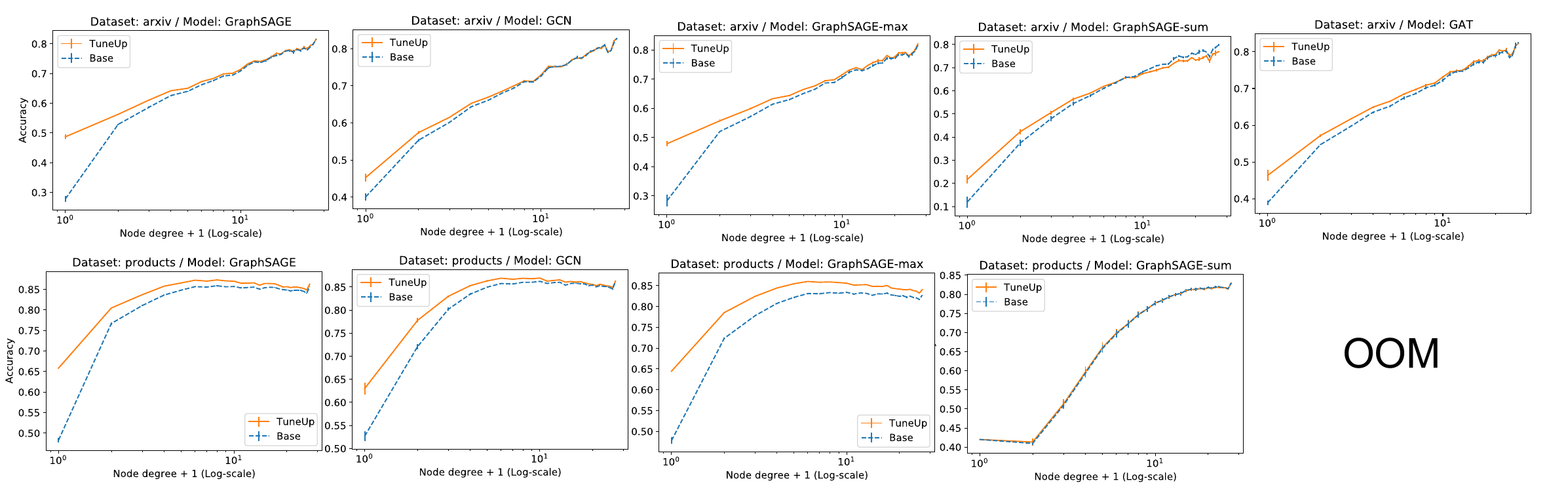}
\caption{
Degree-specific predictive performance of the base GNN and \methodname{} in transductive semi-supervised node classification. The evaluation metric is classification accuracy.
}
\label{fig:nodecls_comp}
\end{figure*}

\begin{figure*}
\centering
\includegraphics[width=1\linewidth]{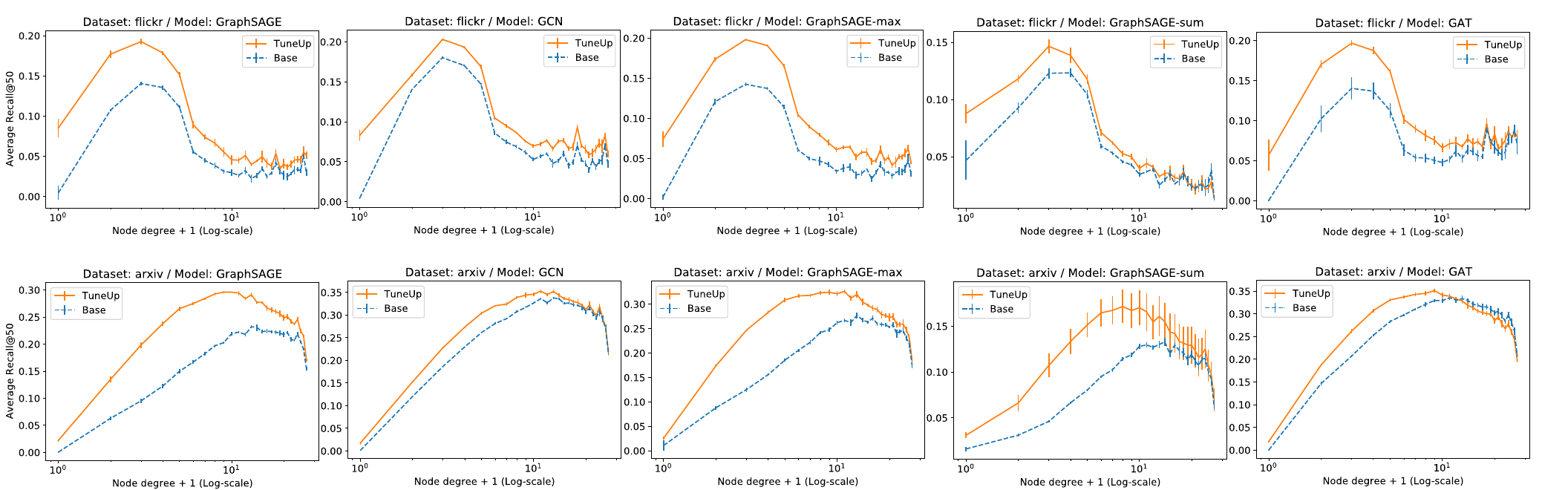}
\caption{
Degree-specific predictive performance of the base GNN and \methodname{} in transductive link prediction. The evaluation metric is recall@50.
}
\label{fig:linkpred_comp}
\end{figure*}

\begin{figure*}
\centering
\includegraphics[width=1\linewidth]{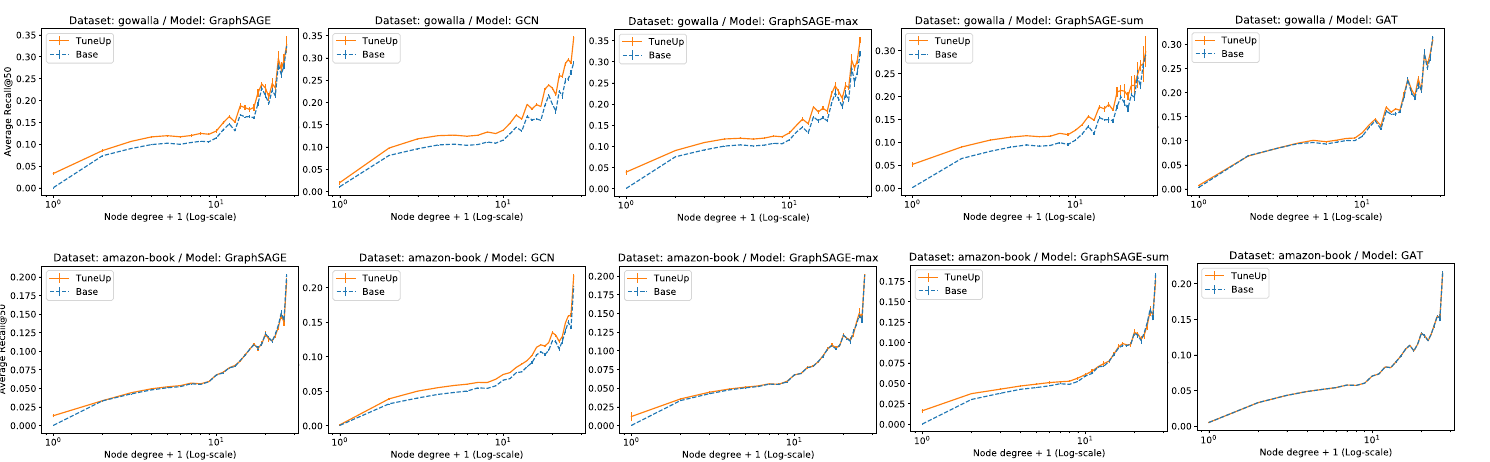}
\caption{
Degree-specific predictive performance of the base GNN and \methodname{} in transductive recommender systems. The evaluation metric is recall@50.
}
\label{fig:recsys_comp}
\end{figure*}

\begin{table*}[t]
    \centering
        \caption{Semi-supervised node classification performance with GCN as the backbone architecture. The evaluation metric is classification accuracy. For the ``Inductive (cold),'' 90\% of edges are randomly removed from the new nodes.
        Refer to Table~\ref{tab:cold_nodecls_gcn_table} in the Appendix for the results with different edge removal ratios.
        }
    \label{tab:nodecls_gcn_table}
    \renewcommand{\arraystretch}{1.0}
    \setlength{\tabcolsep}{5pt}
    \resizebox{\linewidth}{!}{
    \begin{tabular}{l|ccc|ccc}
      \toprule
         
     \mr{2}{\textbf{Method}} & \mc{3}{c|}{\textbf{arxiv}} & \mc{3}{c}{\textbf{products}} \\ \cmidrule{2-7}
        & \textbf{Transductive}  & \textbf{Inductive} & \textbf{Inductive (cold)} & \textbf{Transductive} & \textbf{Inductive} & \textbf{Inductive (cold)} \\
      \midrule
Base & 0.6921\std{0.0004} & 0.6893\std{0.0021} & 0.5491\std{0.0042} & 0.8432\std{0.0007} & 0.8447\std{0.0008} & 0.7461\std{0.0033} \\ 
DropEdge & 0.6958\std{0.0007} & 0.6938\std{0.0013} & 0.5632\std{0.0019} & 0.8486\std{0.0004} & 0.8495\std{0.0008} & 0.7661\std{0.0019} \\ 
LocalAug & \textbf{0.6986\std{0.0010}} & 0.6963\std{0.0022} & 0.5680\std{0.0042} & 0.8485\std{0.0005} & 0.8497\std{0.0002} & 0.7505\std{0.0025} \\ 
ColdBrew & 0.6862\std{0.0004} & 0.6726\std{0.0031} & 0.5376\std{0.0076} & 0.8402\std{0.0008} & 0.8412\std{0.0006} & 0.7396\std{0.0031} \\ 
GraphLessNN & 0.6129\std{0.0008} & 0.5462\std{0.0021} & 0.5462\std{0.0021} & 0.6671\std{0.0008} & 0.6649\std{0.0006} & 0.6649\std{0.0006} \\ 
RAWLS-GCN & 0.6706\std{0.0013} & 0.6696\std{0.0027} & 0.5326\std{0.0027} & 0.8210\std{0.0008} & 0.8223\std{0.0009} & 0.7113\std{0.0013} \\ 
Tail-GNN & 0.6434\std{0.0010} & 0.6402\std{0.0007} & 0.5381\std{0.0054} & OOM & OOM & OOM \\ 
\midrule
\methodname{} w/o curriculum & 0.6960\std{0.0009} & 0.6927\std{0.0023} & 0.5605\std{0.0031} & 0.8482\std{0.0006} & 0.8489\std{0.0004} & 0.7640\std{0.0023} \\ 
\methodname{} w/o pseudo-labels & 0.6965\std{0.0008} & 0.6935\std{0.0015} & 0.5567\std{0.0035} & 0.8490\std{0.0007} & 0.8497\std{0.0009} & 0.7670\std{0.0036} \\ 
\methodname{} w/o syn-tails & 0.6936\std{0.0006} & 0.6924\std{0.0027} & 0.5590\std{0.0053} & 0.8452\std{0.0006} & 0.8467\std{0.0006} & 0.7550\std{0.0024} \\ 
\textbf{\methodname{} (ours)} & \textbf{0.6989\std{0.0006}} & \textbf{0.6990\std{0.0019}} & \textbf{0.5916\std{0.0044}} & \textbf{0.8523\std{0.0007}} & \textbf{0.8536\std{0.0006}} & \textbf{0.7924\std{0.0050}} \\ 
\midrule
Rel. gain over base & +1.0\% & +1.4\% & +7.7\% & +1.1\% & +1.1\% & +6.2\% \\ 
    \bottomrule
    \end{tabular}
    }
\end{table*}

\begin{table*}[t]
    \centering
        \caption{Cold-start inductive node classification performance with GraphSAGE as the backbone architecture. The larger the edge removal ratio is, the more cold-start the prediction task becomes. The evaluation metric is classification accuracy. Refer to Table \ref{tab:cold_nodecls_gcn_table} in Appendix for the performance with GCN, where a similar trend is observed.}
    \label{tab:cold_nodecls_sage_table}
    \renewcommand{\arraystretch}{1.0}
    \setlength{\tabcolsep}{5pt}
    \resizebox{\linewidth}{!}{
    \begin{tabular}{l|ccc|ccc}
      \toprule
     \mr{3}{\textbf{Method}} & \mc{3}{c|}{\textbf{arxiv}} & \mc{3}{c}{\textbf{products}} \\ \cmidrule{2-7}
      & \mc{3}{c|}{\textbf{Edge removal ratio}} & \mc{3}{c}{\textbf{Edge removal ratio}} \\ 
        & 30\% & 60\%  & 90\%  & 30\% & 60\%  & 90\%  \\
      \midrule
Base & 0.6450\std{0.0023} & 0.5993\std{0.0013} & 0.4752\std{0.0061} & 0.8334\std{0.0008} & 0.8130\std{0.0008} & 0.7227\std{0.0011} \\ 
DropEdge & 0.6534\std{0.0017} & 0.6248\std{0.0041} & 0.5449\std{0.0059} & 0.8411\std{0.0007} & 0.8281\std{0.0005} & 0.7709\std{0.0014} \\ 
LocalAug & 0.6547\std{0.0016} & 0.6149\std{0.0011} & 0.4981\std{0.0018} & 0.8370\std{0.0006} & 0.8166\std{0.0010} & 0.7261\std{0.0008} \\ 
ColdBrew & 0.6283\std{0.0035} & 0.5923\std{0.0017} & 0.5082\std{0.0018} & 0.8309\std{0.0008} & 0.8134\std{0.0007} & 0.7395\std{0.0019} \\ 
GraphLessNN & 0.5456\std{0.0008} & 0.5456\std{0.0008} & 0.5456\std{0.0008} & 0.6648\std{0.0009} & 0.6648\std{0.0009} & 0.6648\std{0.0009} \\ 
Tail-GNN & 0.6389\std{0.0011} & 0.6123\std{0.0023} & 0.5388\std{0.0031} & OOM & OOM & OOM \\ 
\midrule
\methodname{} w/o curriculum & 0.6531\std{0.0023} & 0.6266\std{0.0033} & 0.5472\std{0.0119} & 0.8396\std{0.0006} & 0.8243\std{0.0007} & 0.7569\std{0.0015} \\ 
\methodname{} w/o pseudo-labels & 0.6498\std{0.0030} & 0.6192\std{0.0038} & 0.5332\std{0.0077} & 0.8405\std{0.0010} & 0.8262\std{0.0016} & 0.7631\std{0.0055} \\ 
\methodname{} w/o syn-tails & 0.6518\std{0.0016} & 0.6106\std{0.0025} & 0.4899\std{0.0047} & 0.8362\std{0.0004} & 0.8162\std{0.0005} & 0.7258\std{0.0011} \\ 
\textbf{\methodname{} (ours)} & \textbf{0.6685\std{0.0022}} & \textbf{0.6504\std{0.0024}} & \textbf{0.5996\std{0.0012}} & \textbf{0.8521\std{0.0005}} & \textbf{0.8432\std{0.0005}} & \textbf{0.8054\std{0.0011}} \\ 
\midrule
Rel. gain over base & +3.6\% & +8.5\% & +26.2\% & +2.2\% & +3.7\% & +11.4\% \\ 
    \bottomrule
    \end{tabular}
    }
\end{table*}

\begin{table*}[t]
    \centering
        \caption{Cold-start inductive node classification performance with GCN as the backbone architecture. The larger the edge removal ratio is, the more cold-start the prediction task becomes. The evaluation metric is classification accuracy. }
    \label{tab:cold_nodecls_gcn_table}
    \renewcommand{\arraystretch}{1.0}
    \setlength{\tabcolsep}{5pt}
    \resizebox{\linewidth}{!}{
    \begin{tabular}{l|ccc|ccc}
      \toprule
         
     \mr{3}{\textbf{Method}} & \mc{3}{c|}{\textbf{arxiv}} & \mc{3}{c}{\textbf{products}} \\ \cmidrule{2-7}
      & \mc{3}{c|}{\textbf{Edge removal ratio}} & \mc{3}{c}{\textbf{Edge removal ratio}} \\ 
        & 30\% & 60\%  & 90\%  & 30\% & 60\%  & 90\%  \\ 
      \midrule
Base & 0.6713\std{0.0019} & 0.6401\std{0.0029} & 0.5491\std{0.0042} & 0.8375\std{0.0008} & 0.8209\std{0.0012} & 0.7461\std{0.0033} \\ 
DropEdge & 0.6756\std{0.0016} & 0.6475\std{0.0026} & 0.5632\std{0.0019} & 0.8435\std{0.0006} & 0.8298\std{0.0008} & 0.7661\std{0.0019} \\ 
LocalAug & 0.6776\std{0.0019} & 0.6489\std{0.0025} & 0.5680\std{0.0042} & 0.8423\std{0.0008} & 0.8261\std{0.0008} & 0.7505\std{0.0025} \\ 
ColdBrew & 0.6513\std{0.0044} & 0.6188\std{0.0047} & 0.5376\std{0.0076} & 0.8338\std{0.0008} & 0.8168\std{0.0013} & 0.7396\std{0.0031} \\ 
GraphLessNN & 0.5462\std{0.0021} & 0.5462\std{0.0021} & 0.5462\std{0.0021} & 0.6649\std{0.0006} & 0.6649\std{0.0006} & 0.6649\std{0.0006} \\ 
RAWLS-GCN & 0.6490\std{0.0016} & 0.6117\std{0.0027} & 0.5326\std{0.0027} & 0.8130\std{0.0007} & 0.7924\std{0.0006} & 0.7113\std{0.0013} \\ 
Tail-GNN & 0.6277\std{0.0012} & 0.6058\std{0.0015} & 0.5381\std{0.0054} & OOM & OOM & OOM \\ 
\midrule
\methodname{} w/o curriculum & 0.6751\std{0.0017} & 0.6453\std{0.0021} & 0.5605\std{0.0031} & 0.8427\std{0.0005} & 0.8288\std{0.0004} & 0.7640\std{0.0023} \\ 
\methodname{} w/o pseudo-labels & 0.6751\std{0.0027} & 0.6451\std{0.0030} & 0.5567\std{0.0035} & 0.8439\std{0.0011} & 0.8299\std{0.0013} & 0.7670\std{0.0036} \\ 
\methodname{} w/o syn-tails & 0.6737\std{0.0015} & 0.6469\std{0.0022} & 0.5590\std{0.0053} & 0.8398\std{0.0006} & 0.8247\std{0.0010} & 0.7550\std{0.0024} \\ 
\textbf{\methodname{} (ours)} & \textbf{0.6815\std{0.0025}} & \textbf{0.6606\std{0.0004}} & \textbf{0.5916\std{0.0044}} & \textbf{0.8489\std{0.0007}} & \textbf{0.8385\std{0.0014}} & \textbf{0.7924\std{0.0050}} \\ 
\midrule
Rel. gain over base & +1.5\% & +3.2\% & +7.7\% & +1.4\% & +2.1\% & +6.2\% \\ 
    \bottomrule
    \end{tabular}
    }
\end{table*}

\begin{table*}[t]
    \centering
        \caption{Link prediction performance with GCN as the backbone architecture. The evaluation metric is recall@50. For the ``Inductive (cold),'' 60\% of edges are randomly removed from the new nodes. Refer to Table~\ref{tab:cold_linkprerd_gcn_table} in the Appendix for the results with different edge removal ratios. }
    \label{tab:linkpred_gcn_table}
    \renewcommand{\arraystretch}{1.0}
    \setlength{\tabcolsep}{5pt}
    \resizebox{\linewidth}{!}{
    \begin{tabular}{l|ccc|ccc}
      \toprule
     \mr{2}{\textbf{Method}} & \mc{3}{c|}{\textbf{flickr}} & \mc{3}{c}{\textbf{arxiv}} \\ \cmidrule{2-7}
        &  \textbf{Transductive} & \textbf{Inductive} & \textbf{Inductive (cold)} & \textbf{Transductive} & \textbf{Inductive} & \textbf{Inductive (cold)} \\
        \midrule
Base & 0.1355\std{0.0007} & 0.1366\std{0.0008} & 0.0863\std{0.0009} & 0.2242\std{0.0005} & 0.2052\std{0.0005} & 0.1185\std{0.0011} \\ 
DropEdge & 0.1479\std{0.0005} & 0.1401\std{0.0014} & 0.1001\std{0.0019} & 0.2394\std{0.0014} & 0.2108\std{0.0012} & 0.1348\std{0.0013} \\ 
LocalAug & 0.1408\std{0.0011} & 0.1430\std{0.0007} & 0.0930\std{0.0009} & 0.2324\std{0.0013} & 0.2136\std{0.0020} & 0.1209\std{0.0007} \\ 
ColdBrew & 0.1177\std{0.0015} & 0.1174\std{0.0029} & 0.0731\std{0.0025} & 0.1978\std{0.0024} & 0.1788\std{0.0037} & 0.0967\std{0.0031} \\ 
RAWLS-GCN & 0.0660\std{0.0020} & 0.0422\std{0.0020} & 0.0406\std{0.0018} & 0.1057\std{0.0017} & 0.0814\std{0.0027} & 0.0409\std{0.0017} \\ 
Tail-GNN & 0.1287\std{0.0017} & 0.1292\std{0.0020} & 0.0872\std{0.0017} & 0.1492\std{0.0012} & 0.1336\std{0.0022} & 0.0812\std{0.0022} \\ 
\midrule
\methodname{} w/o curriculum & 0.1486\std{0.0009} & 0.1434\std{0.0024} & 0.1027\std{0.0015} & 0.2401\std{0.0022} & 0.2121\std{0.0023} & 0.1328\std{0.0008} \\ 
\methodname{} w/o syn-tails & 0.1395\std{0.0022} & 0.1408\std{0.0024} & 0.0899\std{0.0020} & 0.2282\std{0.0026} & 0.2096\std{0.0032} & 0.1181\std{0.0011} \\ 
\textbf{\methodname{} (ours)} & \textbf{0.1577\std{0.0011}} & \textbf{0.1510\std{0.0016}} & \textbf{0.1072\std{0.0011}} & \textbf{0.2527\std{0.0017}} & \textbf{0.2239\std{0.0027}} & \textbf{0.1428\std{0.0021}} \\ 
\midrule
Rel. gain over base & +16.4\% & +10.6\% & +24.2\% & +12.7\% & +9.1\% & +20.5\% \\ 
    \bottomrule
    \end{tabular}
    }
\end{table*}

\begin{table*}[t]
    \centering
        \caption{Cold-start inductive link prediction performance with GraphSAGE. The evaluation metric is recall@50. The larger the edge removal ratio is, the more cold-start the prediction task becomes. Refer to Table \ref{tab:cold_linkprerd_gcn_table} in Appendix for the performance with GCN, where a similar trend is observed. }
    \label{tab:cold_linkpred_sage_table}
    \renewcommand{\arraystretch}{1.0}
    \setlength{\tabcolsep}{5pt}
    \resizebox{\linewidth}{!}{
    \begin{tabular}{l|ccc|ccc}
      \toprule
     \mr{3}{\textbf{Method}} & \mc{3}{c|}{\textbf{flickr}} &  \mc{3}{c}{\textbf{arxiv}} \\\cmidrule{2-7}
      & \mc{3}{c|}{\textbf{Edge removal ratio}} & \mc{3}{c}{\textbf{Edge removal ratio}} \\ 
        & 30\% & 60\%  & 90\% & 30\% & 60\%  & 90\%  \\
      \midrule
Base & 0.0809\std{0.0013} & 0.0582\std{0.0014} & 0.0173\std{0.0038} & 0.0990\std{0.0032} & 0.0688\std{0.0020} & 0.0208\std{0.0011} \\ 
DropEdge & 0.1136\std{0.0013} & 0.0992\std{0.0008} & 0.0594\std{0.0012} & 0.1529\std{0.0036} & 0.1189\std{0.0046} & 0.0570\std{0.0034} \\ 
LocalAug & 0.0899\std{0.0051} & 0.0646\std{0.0059} & 0.0250\std{0.0159} & 0.1038\std{0.0042} & 0.0734\std{0.0036} & 0.0266\std{0.0062} \\ 
ColdBrew & 0.0547\std{0.0057} & 0.0369\std{0.0045} & 0.0266\std{0.0057} & 0.0898\std{0.0035} & 0.0640\std{0.0031} & 0.0331\std{0.0027} \\ 
Tail-GNN & 0.0663\std{0.0018} & 0.0657\std{0.0016} & 0.0529\std{0.0075} & 0.0725\std{0.0031} & 0.0586\std{0.0031} & 0.0371\std{0.0036} \\ 
\midrule
\methodname{} w/o curriculum & 0.1184\std{0.0020} & 0.1014\std{0.0018} & 0.0622\std{0.0020} & 0.1488\std{0.0049} & 0.1144\std{0.0041} & 0.0535\std{0.0045} \\ 
\methodname{} w/o syn-tails & 0.0800\std{0.0022} & 0.0583\std{0.0013} & 0.0173\std{0.0038} & 0.1017\std{0.0032} & 0.0728\std{0.0032} & 0.0301\std{0.0099} \\ 
\textbf{\methodname{} (ours)} & \textbf{0.1259\std{0.0051}} & \textbf{0.1119\std{0.0069}} & \textbf{0.0734\std{0.0084}} & \textbf{0.1574\std{0.0021}} & \textbf{0.1241\std{0.0025}} & \textbf{0.0598\std{0.0028}} \\ 
\midrule
Rel. gain over base & +55.6\% & +92.2\% & +324.3\% & +59.0\% & +80.4\% & +187.5\% \\ 
    \bottomrule
    \end{tabular}
    }
\end{table*}

\begin{table*}[htb!]
    \centering
        \caption{Cold-start inductive link prediction performance with GCN. The evaluation metric is recall@50. The larger the edge removal ratio is, the more cold-start the prediction task becomes. }
    \label{tab:cold_linkprerd_gcn_table}
    \renewcommand{\arraystretch}{1.0}
    \setlength{\tabcolsep}{5pt}
    \resizebox{\linewidth}{!}{
    \begin{tabular}{l|ccc|ccc}
      \toprule
\mr{3}{\textbf{Method}} & \mc{3}{c|}{\textbf{flickr}}  &  \mc{3}{c}{\textbf{arxiv}} \\ \cmidrule{2-7}
       & \mc{3}{c|}{\textbf{Edge removal ratio}} & \mc{3}{c}{\textbf{Edge removal ratio}} \\
         & 30\% & 60\%  & 90\% & 30\% & 60\%  & 90\%  \\
      \midrule
Base & 0.1137\std{0.0012} & 0.0863\std{0.0009} & 0.0256\std{0.0006} & 0.1701\std{0.0009} & 0.1185\std{0.0011} & 0.0381\std{0.0006} \\ 
DropEdge & 0.1224\std{0.0014} & 0.1001\std{0.0019} & 0.0620\std{0.0031} & 0.1816\std{0.0010} & 0.1348\std{0.0013} & \textbf{0.0591\std{0.0016}} \\ 
LocalAug & 0.1192\std{0.0012} & 0.0930\std{0.0009} & 0.0298\std{0.0023} & 0.1763\std{0.0016} & 0.1209\std{0.0007} & 0.0366\std{0.0012} \\ 
ColdBrew & 0.0948\std{0.0032} & 0.0731\std{0.0025} & 0.0542\std{0.0019} & 0.1458\std{0.0031} & 0.0967\std{0.0031} & 0.0383\std{0.0045} \\ 
RAWLS-GCN & 0.0404\std{0.0021} & 0.0406\std{0.0018} & 0.0397\std{0.0014} & 0.0625\std{0.0019} & 0.0409\std{0.0017} & 0.0247\std{0.0019} \\ 
Tail-GNN & 0.1093\std{0.0016} & 0.0872\std{0.0017} & 0.0511\std{0.0029} & 0.1123\std{0.0023} & 0.0812\std{0.0022} & 0.0421\std{0.0041} \\ 
\midrule
\methodname{} w/o curriculum & 0.1240\std{0.0023} & 0.1027\std{0.0015} & 0.0638\std{0.0014} & 0.1812\std{0.0014} & 0.1328\std{0.0008} & 0.0547\std{0.0037} \\ 
\methodname{} w/o syn-tails & 0.1186\std{0.0025} & 0.0899\std{0.0020} & 0.0274\std{0.0029} & 0.1716\std{0.0019} & 0.1181\std{0.0011} & 0.0357\std{0.0018} \\ 
\textbf{\methodname{} (ours)} & \textbf{0.1292\std{0.0013}} & \textbf{0.1072\std{0.0011}} & \textbf{0.0677\std{0.0030}} & \textbf{0.1920\std{0.0020}} & \textbf{0.1428\std{0.0021}} & \textbf{0.0610\std{0.0026}} \\ 
\midrule
Rel. gain over base & +13.7\% & +24.2\% & +164.1\% & +12.9\% & +20.5\% & +60.3\% \\ 
    \bottomrule
    \end{tabular}
    }
\end{table*}

%%%%%%%%%%%%%%%%%%%%%%%%%%%%%%%%%%%%%%%%%%%%%%
\clearpage
\section{Theoretical Analysis} \label{app:proof}
\subsection{Technical Details}
The generalization improvement on low-degree nodes is expected to happen when the label distribution for each node is invariant w.r.t. the degree of nodes. This condition is captured by the following generating process for the considered part of a graph. A finite set $\Zcal$  of (arbitrarily) large size that consists of $(x,y)$ pairs are sampled accordingly to some unknown distribution without graph structure first. Then,
  $(\bx_i,\by_i)_{i=1}^T$ and $(\tx_i,\ty_i)_{i=1}^R$ are sampled uniformly from $\Zcal$ without replacement, where $(\bx_i,\by_i)_{i=1}^T$ and $(\tx_i,\ty_i)_{i=1}^R$ are used in the zero degree
nodes and the full degree nodes in the part of a graph $G$, respectively. Define $A$ and $B$ to be the sets of node indices of the zero degree and the full degrees, respectively. The labeled node indices for the full degree are sampled uniformly  from $B$ and its set is denoted by $S=\{i \in B : i\text{-th node is in the labeled training dataset}\}$. Let $\theta$ be the parameter trained with a set of labeled nodes $S$. We use nodes $A$ as test data. 

We consider a $K$-layer GNN of a standard form:
$f(X)=H_K \in \RR^{n}$ with $H_{k}=\sigma_{k}(J H_{k-1}W_{k}+ 1_nb_k)$,
where $n$ is the number of all nodes, $\sigma_K = \signn$, $\sigma_k$ represents the ReLU nonlinear function for $k \neq K$, $1_n$ is the column vector of size $n$ with all entries being ones, $(W_{k}, b_k)$ are the learnable parameters included in $\theta$, $H_0=X$, $W_{K} \in \RR^{d \times 1}$, and $b_K \in \RR$. Here, $J  \in \RR^{n \times n}$ is defined by
$ 
J = \mathbf{A} + \mathbf{I},  
$
where $ \mathbf{A} \in \RR^{n \times n}$ denotes the graph adjacency matrix and  $\mathbf{I} \in \RR^{n \times n}$ is the identity matrix. 

Define  $\ell(i)$ and  $\Lcal(i)$  to be   the 0-1 losses of $i$-th node with and without dropping edges: i.e.,  $\ell(i)$ is the 0-1 loss of $i$-th node with the model $F_\theta(\widetilde{G})$ using the modified graph $\widetilde{G}$  that drops all edges for $i \in B$, whereas  $\Lcal(i)$  is the loss with the original graph. 
  Define $\tell$ by $\tell(i)=\ell(i)$ for $i\in S$ and $\tell(i)$ is the loss with the pseudo label for $i \in B\setminus S$. The function $\tLcal$ is defined similarly for $\Lcal$ with the  pseudo label but without dropping edges. Define the average training loss   at the end of the 1st stage of the curriculum-based training by 
$
Q= \frac{1}{|S|} \sum_{t \in S} \Lcal_{1}(t),
$
where $\Lcal_1(i)$ is the 0-1 loss of $i$-th node   at the end of the 1st stage.
For  methods $M_1$--$M_3$, we define the generalization gap between  the test errors of nodes with low degrees and the training errors of nodes with high degrees by $\Delta(M_{1})=\frac{1}{|A|}\sum_{i\in A}\Lcal(i)-\frac{1}{|S|}\sum_{i\in S}\ell(i)$, $\Delta(M_{2})=\frac{1}{|A|}\sum_{i\in A}\Lcal(i)-\frac{1}{|B|}\sum_{i\in B}\tell(i)$, $\Delta(M_{3})=\frac{1}{|A|}\sum_{i\in A}\Lcal(i)-\frac{1}{|B|}\sum_{i\in B}\tLcal(i)$.

\begin{theorem}[A more detailed version of Theorem \ref{thm:1}] \label{thm:2}
 For any $\delta>0$, with probability at least  $1-\delta$, the following holds for all $M \in \{M_1,M_2,M_{3}\}$: 
 
\begin{align*}
\nonumber \Delta(M) & \le\sqrt{\frac{ \one\{M=M_1\}8( d-1)\ln(\frac{16 e |S|}{\delta})+8\ln(\frac{16 e |S|}{\delta})}{|S|}}
 +\one\{M\neq M_1\} Q+\one\{M= M_{3}\}\tau+G,
\end{align*} 
where $\tau=\frac{1}{|B|} \sum_{i\in B}\left(\ell(i) -\Lcal(i)\right)$ and  $G=  \sqrt{\frac{8 d\ln(16 e R/ \delta)}{R}} + \sqrt{\frac{\ln (4/\delta)}{2T}}$.
\end{theorem}
\subsection{Proof of Theorem \ref{thm:2}}

Recall the following lemma from \citep[Theorem 4]{hoeffding1963probability}:
\begin{lemma}[\citealp{hoeffding1963probability}] \label{lemma:1}
Let $\Xcal$ be a finite population of $N$ real points, $X_1,\dots,X_n$ denote a random sample without replacement drawn uniformly from $\Xcal$, and $\bar X_1,\dots,\bar X_n$ denote a random sample with replacement drawn uniformly from $\Xcal$. If $g:\RR \rightarrow \RR$ is continuous and convex,
$$
\EE\left[g\left(\sum_{i=1}^nX_i\right)\right] \le \EE\left[g\left(\sum_{i=1}^n \bar X_i\right)\right].
$$
\end{lemma}
We utilize this lemma in our proof. 

\begin{proof}[Proof of Theorem \ref{thm:2}]

Without loss of generality, we order the node index such that this part of a graph with the  $T$ nodes  with zero degree and the  $R$ nodes with full  degree comes first in node index ordering: i.e., $i$-th node is in the group with zero degree for $i \in A=\{1,\dots,T\}$ and with full degree for $i \in B=\{T+1,\dots,T+R\}$. Define $\Delta_1 = \Delta(M_{1})$ and $\Delta_2 = \Delta(M_{2})$. Since $\frac{1}{|A|}\sum_{i\in A}\Lcal(i)=\frac{1}{|A|}\sum_{i\in A}\ell(i)$,
\begin{align*}
&\Delta_1=\frac{1}{|A|}\sum_{i\in A}\ell(i)-\frac{1}{|S|}\sum_{i\in S}\ell(i) \\ &\Delta_2=\frac{1}{|A|}\sum_{i\in A}\ell(i)-\frac{1}{|B|}\sum_{i\in B}\tell(i). \\ & \Delta_3 =\frac{1}{|A|}\sum_{i\in A}\ell(i)-\frac{1}{|B|}\sum_{i\in B}\tLcal(i) \end{align*}
Since $\Zcal$ is finite, we can write $\Zcal=\{z_i:i\in [N]\}$ where $z_i=(x_i,y_i)$ and $[N]=\{1,2,\dots, N\}$ for some (arbitrarily  large) $N$. Then,  $(\bx_i,\by_i)$ and  $(\tx_i,\ty_i)$  can be equivalently defined as follows: we define   $(\bx_i,\by_i)$ by setting its value  to be $(x_{t_i},y_{t_i})$ (for $i=1\dots,T$) and  define $(\tx_i,\ty_i)$ by setting its value  to be $(x_{r_i},y_{r_i})$ for $i=1\dots,R$, where  $t_1,\dots,t_T$ and $r_1,\dots,r_R$ are sampled sampled uniformly  from $[N]$ without replacement. Define a loss function with node features and label by $\phi(\bz_i)=\ell(i)$ where $\bz_i =(\bx_i,\by_i)=(x_{t_i},y_{t_i})=z_{t_{i}}$ for .   With this, since $\frac{1}{|A|}\sum_{i\in A}\phi(\bz_i )=\frac{1}{|A|}\sum_{i\in A}\phi(z_{t_{i}})=\frac{1}{T}\sum_{i=1}^{T}\phi(z_{t_{i}})$,
\begin{align*}
&\Delta_1=\frac{1}{T}\sum_{i=1}^{T}\phi(z_{t_{i}})-\frac{1}{|S|}\sum_{i\in S}\ell(i) 
\\&\Delta_2=\frac{1}{T}\sum_{i=1}^{T}\phi(z_{t_{i}})-\frac{1}{|B|}\sum_{i\in B}\tell(i).
\\&\Delta_3=\frac{1}{T}\sum_{i=1}^{T}\phi(z_{t_{i}})-\frac{1}{|B|}\sum_{i\in B}\tLcal(i).
\end{align*}

\citet{hoeffding1963probability} shows that using Lemma \ref{lemma:1} within the proof of  (standard) Hoeffding's inequality,  (standard)  Hoeffding's inequality still holds for samples without replacement. Since $\ell(i)\in[0,1]$ for all $i \in [N]$,    
$$
\PP\left(\frac{1}{T}\sum_{i=1}^{T}\phi(z_{t_{i}})-\frac{1}{N}\sum_{i=1}^{N}\phi(z_{i})  \ge \epsilon \right)\le \exp\left(-2T \epsilon^2\right).
$$
By setting $\delta=\exp\left(-2T \epsilon^2\right)$ and solving for $\epsilon$, this implies that for any $\delta>0$, with probability at least  $1-\delta$, 
\begin{align*} 
\frac{1}{T}\sum_{i=1}^{T}\phi(z_{t_{i}})\le\frac{1}{N}\sum_{i=1}^{N}\phi(z_{i})  + \sqrt{\frac{\ln (1/\delta)}{2T}}.
\end{align*}
Therefore,  for any $\delta>0$, with probability at least  $1-\delta$, 
\begin{align*}
&\Delta_1\le \frac{1}{N}\sum_{i=1}^{N}\phi(z_{i})  -\frac{1}{|S|}\sum_{i\in S}\ell(i)+ \sqrt{\frac{\ln (1/\delta)}{2T}}, \text{ and, }
\\&\Delta_2\le\frac{1}{N}\sum_{i=1}^{N}\phi(z_{i})  -\frac{1}{|B|}\sum_{i\in B}\tell(i)+ \sqrt{\frac{\ln (1/\delta)}{2T}}.
\\ & \Delta_3\le\frac{1}{N}\sum_{i=1}^{N}\phi(z_{i})  -\frac{1}{|B|}\sum_{i\in B}\tLcal(i)+ \sqrt{\frac{\ln (1/\delta)}{2T}}.
\end{align*}
For $\Delta_1$, we have 
\begin{align*}
\Delta_1 &\le \frac{1}{N}\sum_{i=1}^{N}\phi(z_{i})  -\frac{1}{|S|}\sum_{i\in S}\ell(i)+ \sqrt{\frac{\ln (1/\delta)}{2T}} \pm \frac{1}{|B|}\sum_{i\in B}\ell(i)
\\ & = \left(\frac{1}{N}\sum_{i=1}^{N}\phi(z_{i})  -\frac{1}{|B|}\sum_{i\in B}\ell(i) \right) + \left(\frac{1}{|B|}\sum_{i\in B}\ell(i)-\frac{1}{|S|}\sum_{i\in S}\ell(i) \right)+ \sqrt{\frac{\ln (1/\delta)}{2T}} 
\\ & = \left(\frac{1}{N}\sum_{i=1}^{N}\phi(z_{i})  -\frac{1}{R}\sum_{i=1}^R\phi(z_{r_{i}}) \right) + \left(\frac{1}{|B|}\sum_{i\in B}\ell(i)-\frac{1}{|S|}\sum_{i\in S}\ell(i) \right)+ \sqrt{\frac{\ln (1/\delta)}{2T}} 
\end{align*}
Similarly, for $\Delta_2$ and $\Delta_3$, 
\begin{align*}
&\Delta_2 \le \left(\frac{1}{N}\sum_{i=1}^{N}\phi(z_{i})  -\frac{1}{R}\sum_{i=1}^R\phi(z_{r_{i}}) \right)   + \left(\frac{1}{|B|}\sum_{i\in B}\ell(i)-\frac{1}{|B|}\sum_{i\in B}\tell(i) \right)+ \sqrt{\frac{\ln (1/\delta)}{2T}} 
\\ & \Delta_3 \le \left(\frac{1}{N}\sum_{i=1}^{N}\phi(z_{i})  -\frac{1}{R}\sum_{i=1}^R\phi(z_{r_{i}}) \right)   + \left(\frac{1}{|B|}\sum_{i\in B}\ell(i)-\frac{1}{|B|}\sum_{i\in B}\tLcal(i) \right)+ \sqrt{\frac{\ln (1/\delta)}{2T}} 
\end{align*}
For the first term of both RHS\ of  $\Delta_1$, $\Delta_2$, and $\Delta_3$, by using  \citep[Proposition 1]{esser2021learning} with the empty graph, we have that  for any $\delta>0$, with probability at least  $1-\delta$, 
\begin{align*} 
 \frac{1}{N}\sum_{i=1}^{N}\phi(z_{i})  -\frac{1}{R}\sum_{i=1}^R\phi(z_{r_{i}})\le \sqrt{\frac{8d \ln(e R)+8\ln(4/\delta)}{R}}\le \sqrt{\frac{ 8 d\ln(\frac{4 e R}{\delta})}{R}} , 
\end{align*}
where we have  $d$ because the rank of the graph aggregation matrix for this empty graph part can be larger than $d$.
For the second term of RHS\ of  $\Delta_1$, similarly by using  \citep[Proposition 1]{esser2021learning} with the empty graph, we have that  for any $\delta>0$, with probability at least  $1-\delta$, 
\begin{align*} 
 \frac{1}{|B|}\sum_{i\in B}\ell(i)-\frac{1}{|S|}\sum_{i\in S}\ell(i)\le \sqrt{\frac{8d \ln(e |S|)+8\ln(4/\delta)}{|S|}}\le \sqrt{\frac{ 8 d\ln(\frac{4 e |S|}{\delta})}{|S|}}. \end{align*}
Therefore, by combining these with union bounds, we have that for any $\delta>0$, with probability at least  $1-\delta$, 
\begin{align} \label{eq:1}
\nonumber &\Delta_1\le  \sqrt{\frac{ 8 d\ln(\frac{12 eR}{\delta})}{R}} +\sqrt{\frac{ 8 d\ln(\frac{12 e |S|}{\delta})}{|S|}}+ \sqrt{\frac{\ln (3/\delta)}{2T}}, \text{ and, }
\\&\Delta_2\le  \sqrt{\frac{ 8 d\ln(\frac{12 e R}{\delta})}{R}}+ \left(\frac{1}{|B|}\sum_{i\in B}\ell(i)-\frac{1}{|B|}\sum_{i\in B}\tell(i) \right)+ \sqrt{\frac{\ln (3/\delta)}{2T}}.
\\&\Delta_3\le  \sqrt{\frac{ 8 d\ln(\frac{12 e R}{\delta})}{R}}+ \left(\frac{1}{|B|}\sum_{i\in B}\ell(i)-\frac{1}{|B|}\sum_{i\in B}\tLcal(i) \right)+ \sqrt{\frac{\ln (3/\delta)}{2T}}.
\end{align}

For the second  term of RHS\ of  $\Delta_2$ and $\Delta_3$,  we formalize and take advantage of the curriculum-based training by relating this second term for the objective of the curriculum-based training and its property. That is, we formalize the fact that the generalization errors over nodes of high-degrees are minimized relatively well when compared to those of low degrees, at the first stage, because of the use of the graph structure at the first stage. Then, such low generalization errors with the full graph information is utilized to reduce the generalization errors for low degree nodes at the second stage. We formalize these intuitions in  our proof. 

Recall  that $\tell(t)=\ell(t)$ for $t\in S$ and $\tell(t)$ is the loss with the pseudo label for $t \in B \setminus S$. 
Thus, we have that 
\begin{align*}
&\frac{1}{|B|}\sum_{i\in B}\ell(i)-\frac{1}{|B|}\sum_{i\in B}\tell(i)=\frac{1}{|B|} \sum_{i\in B}\left(\ell(i) - \tell(i)\right).
\\ & \frac{1}{|B|}\sum_{i\in B}\ell(i)-\frac{1}{|B|}\sum_{i\in B}\tLcal(i)=\frac{1}{|B|} \sum_{i\in B}\left(\Lcal(i) - \tLcal(i)\right)+\frac{1}{|B|} \sum_{i\in B}\left(\ell(i) -\Lcal(i)\right).
\end{align*}
Here, we observe that   $(\ell(i)-\tell(i))=0$ if the pseudo label of $i$-th node is correct. In other words,   $(\ell(i)-\tell(i))=0$  if $\Lcal_1(i)=0$, where $\Lcal_1(i)$ is the 0-1 loss of $i$-th node with the  model $F_{\htheta}(G)$ using the original graph $G$ \textit{where $\htheta$ is fixed at the end of the first stage of the curriculum-based training} (with or without dropping edges in the second stage). Since $(\ell(i)-\tell(i)) \le 1$ and $\Lcal_1(i) \in \{0,1\}$, this implies that $$
\ell(i)-\tell(i)\le\Lcal_1(i). 
$$
Similarly,
$$
\Lcal(i) - \tLcal(i) \le\Lcal_1(i). 
$$
Combining these,  
\begin{align*}
&\Delta_2 \le   \sqrt{\frac{ 8 d\ln(\frac{12 e R}{\delta})}{R}}+\frac{1}{|B|} \sum_{i\in B}\Lcal_1(i)+ \sqrt{\frac{\ln (3/\delta)}{2T}} 
\\ & \Delta_3 \le   \sqrt{\frac{ 8 d\ln(\frac{12 e R}{\delta})}{R}}+\frac{1}{|B|} \sum_{i\in B}\Lcal_1(i)+\frac{1}{|B|} \sum_{i\in B}\left(\ell(i) -\Lcal(i)\right)+ \sqrt{\frac{\ln (3/\delta)}{2T}}
\end{align*}
 Here, by invoking \citep[Proposition 1]{esser2021learning} with the original graph $G$, we have that for any $\delta>0$, with probability at least  $1-\delta$, 
$$
\frac{1}{|B|} \sum_{i\in B}\Lcal_1(i)\le\frac{1}{|S|} \sum_{t \in S} \Lcal_{1}(t)+\sqrt{\frac{ 8\ln(\frac{4 e |S|}{\delta})}{|S|}}, 
$$
where we can remove $d$ because the rank of the graph aggregation matrix for this part of the loss is one. By combining these with \eqref{eq:1} via  union bounds, we have that for any $\delta>0$, with probability at least  $1-\delta$, 
\begin{align*}
\nonumber &\Delta_1\le  \sqrt{\frac{ 8 d\ln(\frac{16 e R}{\delta})}{R}} +\sqrt{\frac{ 8 d\ln(\frac{16 e |S|}{\delta})}{|S|}}+ \sqrt{\frac{\ln (4/\delta)}{2T}}, 
\\&\Delta_2\le  \sqrt{\frac{ 8 d\ln(\frac{16 e R}{\delta})}{R}}+\sqrt\frac{ 8\ln(\frac{16 e |S|}{\delta})}{|S|}+\frac{1}{|S|} \sum_{t \in S} \Lcal_{1}(t)+ \sqrt{\frac{\ln (4/\delta)}{2T}}, \text{ and, }
\\&\Delta_3\le  \sqrt{\frac{ 8 d\ln(\frac{16 e R}{\delta})}{R}}+\sqrt\frac{ 8\ln(\frac{16 e |S|}{\delta})}{|S|}+\frac{1}{|S|} \sum_{t \in S} \Lcal_{1}(t)+ \sqrt{\frac{\ln (4/\delta)}{2T}}+\frac{1}{|B|} \sum_{i\in B}\left(\ell(i) -\Lcal(i)\right).
\end{align*}

\end{proof}

\end{document}